\documentclass[final,breaklinks]{colt2020-plain}
\usepackage{times}
\usepackage{booktabs}
\usepackage{multirow}
\usepackage{framed}
\usepackage{amsmath,amssymb,rotating,graphicx,rotating,float}
\usepackage{url}

\usepackage{color}

\usepackage{cleveref}
\newlength{\dhatheight}

\newcommand{\alg}{\mathbb{A}}
\renewcommand{\H}{\mathbb H}
\newcommand{\C}{\mathbb C}
\newcommand{\SOA}{{\rm SOA}}
\newcommand{\QC}{{\rm QC}}
\newcommand{\MB}{{\rm MB}}
\newcommand{\vc}{{\rm {V}}}
\newcommand{\dvc}{\vc^*}
\newcommand{\helly}{{\rm K}}
\newcommand{\LD}{{\rm L}}
\newcommand{\Maj}{{\rm Maj}}
\newcommand{\Hyps}{Q}
\newcommand{\Vote}{{\rm Vote}}
\newcommand{\HighVote}{{\rm HighVote}}
\newcommand{\val}{{\rm val}}
\newcommand{\A}{\mathcal{A}}
\newcommand{\B}{\mathcal{B}}
\newcommand{\tx}{\tilde{x}}
\newcommand{\ty}{\tilde{y}}
\newcommand{\relconst}{c_1}
\newcommand{\unifconst}{c_2}
\newcommand{\vcconst}{c_0}

\renewcommand{\epsilon}{\varepsilon}
\newcommand{\eps}{\varepsilon}

\newcommand{\X}{\mathcal X}

\newcommand{\Y}{\mathcal Y}

\newcommand{\Z}{\mathcal{Z}}
\newcommand{\target}{h^{*}}

\DeclareSymbolFont{bbold}{U}{bbold}{m}{n}
\DeclareSymbolFontAlphabet{\mathbbold}{bbold}
\newcommand{\ind}{\mathbbold{1}}

\newcommand{\F}{\mathcal{F}}

\renewcommand{\S}{\mathcal S}

\newcommand{\I}{\mathcal{I}}
\newcommand{\nats}{\mathbb{N}}

\newcommand{\E}{\mathop{\mathbb{E}}}

\newcommand{\argmax}{\mathop{\rm argmax}}
\newcommand{\argmin}{\mathop{\rm argmin}}

\newcommand{\supp}{{\rm supp}}

\newcommand{\ignore}[1]{}
\newcommand{\oldstuff}[1]{}

\colorlet{sgreen}{black!45!green}

\newcommand{\expert}{g}
\newcommand{\Experts}{\mathbb{G}}

\newsavebox{\savepar}
\newenvironment{bigboxit}{\begin{center}\begin{lrbox}{\savepar}
\begin{minipage}[h]{5.2in}
\normalfont
\begin{flushleft}}
{\end{flushleft}\end{minipage}\end{lrbox}\fbox{\usebox{\savepar}}
\end{center}}

\makeatletter
\newcommand{\vast}{\bBigg@{3}}
\newcommand{\Vast}{\bBigg@{4}}
\makeatother

\renewenvironment{proof}[1][]{\par\noindent{\bf Proof #1\ }}{\hfill\BlackBox\\[2mm]}

\title[Online Learning with Simple Predictors]{Online Learning with Simple Predictors and \\
a Combinatorial Characterization of Minimax in 0/1 Games}

\coltauthor{%
\Name{Steve Hanneke} \Email{steve.hanneke@gmail.com} \\ \addr Toyota Technological Institute at Chicago 
\AND
\Name{Roi Livni}
\Email{rlivni@tauex.tau.ac.il}\\
\addr Tel Aviv University
 \AND
 \Name{Shay Moran} \Email{smoran@technion.ac.il} \\ \addr Technion
}

\begin{document}
\maketitle

\begin{abstract}
Which classes can be learned properly in the online model?
    --- that is, by an algorithm that on each round uses a predictor from the concept class.
    While there are simple and natural cases where improper learning is useful and even necessary,
    it is natural to ask how complex must the improper predictors be in such cases.
    Can one always achieve nearly optimal mistake/regret bounds using ``simple'' predictors?

In this work, we give a complete characterization of when this is possible, 
    thus settling an open problem which has been studied since the pioneering works of
    Angluin (1987) and Littlestone (1988).
    More precisely, given any concept class $\C$ and any hypothesis class $\H$
    we provide nearly tight bounds (up to a log factor) 
    on the optimal mistake
    bounds for online learning $\C$ using predictors from $\H$.
    Our bound yields an exponential improvement over the previously best known bound by Chase and Freitag (2020).

As applications, we give constructive proofs showing that 
    (i) in the realizable setting, a near-optimal mistake bound (up to a constant factor) 
    can be attained by a sparse majority-vote of proper predictors, and
    (ii) in the agnostic setting, a near optimal regret bound (up to a log factor)
    can be attained by a randomized proper algorithm.
    The latter was proven non-constructively by Rakhlin, Sridharan, and Tewari (2015).
    It was also achieved by constructive but improper algorithms proposed by Ben-David, Pal, and Shalev-Shwartz (2009) and Rakhlin, Shamir, and Sridharan (2012).

A technical ingredient of our proof which may be of independent interest
    is a generalization of the celebrated Minimax Theorem (von Neumann, 1928)
    for binary zero-sum games with arbitrary action-sets:
    a simple game which fails to satisfy
    Minimax is 
    ``{\it Guess the Larger Number}''.
    In this game, each player picks a natural number 
    and the player who picked the larger number wins.
    Equivalently, the payoff matrix of this game is infinite triangular.
    We show that this is the only obstruction:
    if the payoff matrix does not contain triangular submatrices 
    of unbounded sizes then the Minimax Theorem is satisfied.
    This generalizes von Neumann's Minimax Theorem
    by removing requirements of finiteness (or compactness) of the action-sets, 
    and moreover it captures precisely the types of 
    games of interest in online learning: 
    namely, Littlestone games.

\end{abstract}


\section{Introduction}

An improper learning algorithm is an algorithm that learns a class $\C$
    using hypotheses $h$ that are not necessarily in $\C$.
    While at a first sight this may seem like a counter-intuitive thing to do,
    improper algorithms are extremely powerful and using them often circumvents
    computational issues and sample complexity barriers \citep*{srebro2005maximum,candes2009exact, Anava13series,Hazan15lowrank,hanneke:16a,Hazan16unsupervised,Hazan17matrix,Agarwal19control}.
    In fact, there are extreme examples of learning tasks 
    that can only be performed by improper algorithms~\citep*{daniely:14,daniely:15,angluin:87,montasser:19}.

However, while many of the improper algorithms proposed in the 
theoretical literature use sophisticated representations 
\citep*[e.g.,][]{haussler:94,littlestone:88}, other 
natural improper algorithms use ``simple'' hypotheses 
    which often can be described as simple combinations of functions from $\C$. 
    For instance, many algorithms (e.g., boosting)
    use sparse weighted majority-votes of concepts from the class.
    It is therefore natural to ask:
\begin{framed}
    \begin{center}
    Can any given (learnable) class $\C$ 
    be learned by algorithms which use ``simple'' hypotheses?
    \end{center}
\end{framed}
    
\paragraph{Online Classification.}
While the above question has been extensively studied (and answered\footnote{Indeed, by {\it Uniform Convergence}, any {\it Empirical Risk Minimizer} attains a near-optimal sample complexity in the PAC model. \citep*{vapnik:71}}) 
    in the batch setting, in the online setting it remains largely open.
    Note that already in the realizable mistake-bound model there are 
    learnable classes which cannot be learned properly:
    one simple example is the class of all singletons over $\mathbb{N}$.
    Indeed, in each round any proper learner must use a singleton $1_{\{n\}}$, 
    and therefore the adversary can always force a mistake by presenting the example~$(n,0)$.
    Note however, that if the learner could use the all-zero function $1_{\emptyset}$
    (which is not in the class), then she would only make one mistake before learning
    the target concept.
    
In general, the optimal mistake-bound for learning $\C$ 
    is achieved by the {\it Standard Optimal Algorithm} (SOA) of Littlestone, 
    which is improper (as the above example shows). 
    It is interesting to note that it is not known whether the hypothesis space $\H$
    used by the SOA is simple in any natural sense: e.g.,\ it is not known whether its Littlestone dimension is bounded in terms of the Littlestone dimension of $\C$, or even whether its VC dimension is.
    (In the above example, the SOA only uses the class $\H=\C\cup\{1_\emptyset\}$
    which has the same Littlestone and VC dimensions as $\C$.)
    One of the results in this work provides a nearly optimal algorithm (up to a numerical multiplicative factor)
    which uses hypotheses that are sparse majority votes of functions from $\C$.

\subsection{Our Contribution}

In this section we survey the main contributions of this work.
    Some of the statements involve standard technical terms which
    are defined in \Cref{sec:main-results}, where we also
    give complete formal statements of our results.

\subsubsection*{Main Result I: When Do Hypotheses From $\H$ Suffice To Learn $\C$?}
Our first main result provides a complete combinatorial characterization of a near-optimal mistake bound
    for online learning $\C$ using hypotheses from $\H$. 
    This settles an open problem which was studied since the early days of 
    Computational Learning Theory~\citep{angluin:87,littlestone:88,angluin:90,hellerstein:96,balcazar:01,balcazar:02,balcazar:02b,Hayashi:03,angluin:20,chase:20}.

\paragraph{Two Basic Lower Bounds.}
The early works of \cite{angluin:87,angluin:04,balcazar:02b} 
and \cite{littlestone:88} presented two basic lower bounds
    for online learning $\C$ using $\H$. 
In the seminal work which introduced the mistake-bound model \cite{littlestone:88} 
    defined the Littlestone\footnote{The name, \emph{Littlestone dimension}, was later coined by \citet*{ben2009agnostic}.} dimension and noticed that it provides a lower bound on the number of mistakes,
    even if the algorithm is allowed to use any hypothesis $h:\X\to \{0,1\}$ (i.e.\ $\H=\{0,1\}^\X$).
 
The second (and less known) lower bound was 
rooted in the seminal work of \citet*{angluin:87} 
    which introduced the equivalence-query model. 
    This bound was later generalized by 
    \citet*{hellerstein:96,balcazar:02b} 
    and today it is known 
    as the {\it strong consistency dimension}~\citep*{balcazar:02b} or the {\it dual Helly number}~\citep*{bousquet:20}.
    The idea behind this lower bound can be seen as 
    a generalization of the argument showing that singletons
    over $\mathbb{N}$ are not properly learnable, which was 
    discussed earlier. 
    Assume there is a set $S \subseteq \X \times \{0,1\}$ 
    of labeled examples such that \underline{no} function 
    $h \in \H$ satisfies $(\forall (x,y) \in S):h(x)=y$, 
    but for some $k \in \nats$ \emph{every} 
    $(x_1,y_1),\ldots,(x_k,y_k) \in S$ has some 
    $c \in \C$ with $(\forall i \leq k):c(x_i)=y_i$.
    In words, $S$ is not realizable by the hypothesis 
    class $\H$, but every subset of $S$ of size $k$ 
    is realizable by the concept class $\C$.
    Given such a set $S$, the adversary can force $k$ 
    mistakes for any learner with hypothesis class $\H$,
    as follows.  On each round, the learner proposes a 
    hypothesis $h \in \H$, and since $S$ is not realizable 
    by $\H$ there must exist $(x,y) \in S$ with 
    $h(x) \neq y$, so that this counts as a mistake.
    Since all subsets of $S$ of size $k$ are realizable 
    by $\C$, this adversary also guarantees that all $k$ 
    examples it gives will still be consistent with some 
    $c \in \C$.
    However, if instead there exists a subset of $S$ 
    of size $k$ that is \underline{not} realizable by $\C$, 
    then this strategy for the adversary might fail.
    The {\it dual Helly number} $\helly$ 
    of $\C$ relative to $\H$ 
    is the smallest $k$ such that, for every set $S$ 
    not realizable by $\H$, there exists a subset of 
    size at most $k$ not realizable by $\C$; 
    $\helly$ is defined to be infinite 
    if no such $k$ exists.  Thus, the above 
    adversary can always force at least $\helly-1$ mistakes.

We show in \Cref{thm:MB} that the Littlestone dimension and dual Helly number are the only obstacles
    for learning $\C$ using $\H$:
\begin{framed}
\vspace{-0.5cm}
\begin{center} {\bf Main Result I (\Cref{thm:MB})}\end{center}
\vspace{-0.45cm}
\noindent There exists a deterministic algorithm which online learns $\C$ with at most
$O(\LD\cdot\helly\cdot\log(\helly))$ mistakes while only using hypotheses from $\H$,
where $\LD$ is the Littlestone dimension of $\C$ and $\helly$ is the dual Helly number of $\C$ relative to~$\H$.
\end{framed}
This improves over the best previous result $\helly^\LD$ by \citep*{chase:20} who were the first to show
that $\C$ can be learned using hypotheses from $\H$ if and only if $\helly,\LD<\infty$. 
Our result further improves the upper bound to a polynomial dependence which is nearly tight 
in the sense that a lower bound of
$\max\{\LD, \helly-1\}$ on the optimal mistake bound always holds, 
and there exist classes $\C,\H$ of any $\LD$ and $\helly$ 
for which the optimal number of 
mistakes is
$\Omega( \LD \helly )$ 
\citep*{angluin:87,littlestone:88,balcazar:02b}.

\subsubsection*{Main Result II: Optimal Mistake-Bounds Using Sparse-Majorities}
We have seen simple examples demonstrating that sometimes 
    one has to use improper learners in order to achieve non-trivial mistake bounds in online learning $\H$. 
A natural question is ``how improper'' must an optimal algorithm be? 
Is there an algorithm which is close to being proper?
Our second main result shows that it is possible to achieve a nearly optimal mistake-bound 
(up to a universal constant factor), using predictors that are sparse majority votes of functions in the class.
\begin{framed}
\vspace{-0.5cm}
\begin{center} {\bf Main Result II (\Cref{thm:majorities})}\end{center}
\vspace{-0.45cm}
\noindent There exists a deterministic algorithm which online learns $\C$ with at most
$O(\LD)$ mistakes while only using hypotheses of the form $\Maj(h_1,\ldots, h_p)$ for $h_i\in \C$,
where $\LD$ is the Littlestone dimension of $\C$, and 
$p$
is a constant depending only on $\C$ (proportional to dual VC dimension).
\end{framed}

This provides another demonstration to the usefulness of majority-votes/ensemble-methods in classification
and might be seen as a kind of online learning analogue 
of Hanneke's result for optimal PAC learning \citep*{hanneke:16a}, 
which is also achievable by an algorithm based on 
majority votes of concepts in $\C$.
A corollary of this result is that there exist randomized 
proper algorithms whose expected mistake-bound is $\tilde O(\eps\cdot T + \frac{\LD}{\eps})$, where $\LD$ is the Littlestone dimension of the class.

\subsubsection*{Main Result III: Near-Optimal Regret-Bounds Using Randomized Proper Algorithms}

In a fascinating result, \citet*{rakhlin2015online} 
    have established the existence of optimal randomized proper algorithms in agnostic online learning 
    (under some additional topological restrictions). 
    Interestingly, their result is non-constructive:
    they prove the existence of such an algorithm by viewing online learning as a repeated game between the learner and the adversary, and by analyzing the value of that game via a dual perspective, which involves an application
    of the Minimax Theorem w.r.t.\ the \emph{repeated} game. 
    On the other hand, \citet*{rakhlin2012relax} and \citet*{ben2009agnostic} 
    gave constructive proofs of this fact, but the implied algorithms are randomized \emph{improper} algorithms.
    The following result achieves the best of both worlds: 
\begin{framed}
\vspace{-0.5cm}
\begin{center} {\bf Main Result III (\Cref{thm:agnostic})}\end{center}
\vspace{-0.45cm}
\noindent We give a constructive proof demonstrating that any class $\C$ can be online learned
in the agnostic setting by a randomized proper 
algorithm whose expected regret is 
$\tilde O(\sqrt{\LD\cdot T})$,
where $\LD$ is the Littlestone dimension, and $T$ is the horizon.
\end{framed}
Note that in the agnostic setting randomization is essential (see e.g.\ \citealp*{ben2009agnostic,cesa-bianchi:06}),
    and unlike the realizable setting, in the agnostic setting
    nearly optimal randomized proper learners can exist.
    Another advantage of our proof compared to the argument by~\citet*{rakhlin2015online}
    is that their application of the Minimax Theorem required additional assumptions
    such as separability and compactness, which are not needed in our proof.

\smallskip

One may view the above result as yet another demonstration of the benefits of randomized algorithms.
Moreover, if one views online learning as a repeated zero-sum game where the learner's pure strategies
are the hypotheses in $\H$, and the adversary's pure strategies are labelled examples,
then a randomized proper online learner \emph{simply corresponds to a mixed strategy in the repeated game},
and so this can be interpreted as another manifestation of the benefits of mixed (randomized)
strategies in online learning.

In fact, the Minimax Theorem for zero-sum games is a key component in our derivation,
and as another technical contribution of this work, 
we prove a generalization of it to infinite games which we discuss next.

\subsubsection*{Main Result IV: A Generalization of the Minimax Theorem}

Consider the following ``guess the larger number'' game between two players 
whom we call Alice and Bob.
Each of the players privately picks a natural number;
then, Alice and Bob reveal the numbers to each other, 
and the winner is the player who picked the larger number.
Note that the payoff matrix of this game is an  $\mathbb{N}\times\mathbb{N}$ triangular matrix.

This game does not satisfy the Minimax theorem.
Indeed, given any mixed strategy $P$ of Alice, namely a distribution over the natural numbers,
Bob can pick a sufficiently large natural number $n$ such that the probability
that a random number $m\sim P$ chosen by Alice satisfies $m\geq n$ is arbitrarily small.
Thus, for every mixed strategy played by Alice, Bob can find a response which wins with probability arbitrarily close to $1$.
By symmetry, also the opposite holds: for every mixed strategy played by Bob,
Alice can find a response that wins with probability arbitrarily close to $1$. 
We show that this game is the \emph{only} obstruction to the Minimax Theorem in the following sense:
\begin{framed}
\vspace{-0.5cm}
\begin{center} {\bf Main Result IV (\Cref{thm:minimax} and \Cref{cor:littlestone-games})}\end{center}
\vspace{-0.45cm}
\noindent The Minimax Theorem applies to every (possibly infinite) binary-valued zero-sum game,
provided that its payoff matrix does not contain triangular\footnote{I.e.,\ $1$ above the diagonal and $0$ below it.}
sub-matrices of unbounded sizes.
\end{framed}
Thus, this result identifies the size of a triangular submatrix as a combinatorial dimension
    which replaces the assumption that the action-sets are finite in the classical Minimax Theorem of von Neumann.

In the context of online learning, this implies that the Minimax Theorem applies to any game 
which corresponds to agnostic online learning for a given Littlestone class $\C$: 
i.e.,\ the learner's strategies are hypotheses 
in $\C$ and the adversary's strategies are labelled examples.
This follows due to the connection between the Littlestone dimension and the {\it Threshold dimension},
which implies that the payoff matrix of this game 
does not contain triangular submatrices of unbounded sizes (\citet*{shelah:78}, see \citet*{AlonLMM19} for an elementary proof using learning theoretic terminology).

\section{Formal Definitions and Main Results}\label{sec:main-results}
\subsection{Definitions and Notation}
\label{nec:notation}

Fix any non-empty set $\X$, known as the \emph{instance space}.  Also 
define $\Y = \{0,1\}$, the \emph{label space}. 
Our results will be expressed in terms of an 
abstract \emph{concept class} $\C$ and \emph{hypothesis class} $\H$, 
which can be any non-empty sets of \emph{concepts}, that is, 
functions $h : \X \to \Y$.

For any set $\H$ of concepts, 
and any sequence $S \in (\X \times \Y)^*$ 
or set $S \subseteq (\X \times \Y)$, 
define 
$\H_{S} = \{ h \in \H : \forall (x,y) \in S, h(x)=y \}$.  
We say $S$ 
is $\H$-realizable if $\H_{S} \neq \emptyset$, 
and otherwise we say $S$ is $\H$-unrealizable.
To simplify notation, also abbreviate 
$\H_{(x,y)} = \H_{\{(x,y)\}}$ 
for $(x,y) \in \X \times \Y$.

For any sequence $x_1,x_2,\ldots$, 
we use the notation $x_{1:t} = (x_1,\ldots,x_t)$, 
or for $(x_1,y_1),(x_2,y_2),\ldots$, we write 
$(x_{1:t},y_{1:t}) = ((x_1,y_1),\ldots,(x_t,y_t))$.
Also, generally define $\log(z) = \max\{\ln(z),1\}$ for $z \geq 1$.

\paragraph{Online Learning.}
An \emph{online learning algorithm} is formally defined as a 
function $\alg : (\X \times \Y)^* \times \X \to \Y$, 
with the interpretation that for a sequence of examples 
$(x_1,y_1),(x_2,y_2),\ldots$, the algorithm's prediction at time $t$ 
is 
\[\hat y_t = \alg(x_{1:(t-1)},y_{1:(t-1)},x_t),\] and we say 
the algorithm makes a \emph{mistake} at time $t$ if 
$\hat y_t \neq y_t$.
We will also write 
\[\hat{h}_t = \alg(x_{1:(t-1)},y_{1:(t-1)})\]
as the function $\X \to \Y$ such that 
$\hat{h}_t(x) = \alg(x_{1:(t-1)},y_{1:(t-1)},x)$.
Generally, for any concept class $\C$, 
learning algorithm $\alg$, 
and $T \in \nats$, define the algorithm's \emph{hypothesis class} 
\[\H(\C,\alg,T) = \{ \alg(S) : S \in (\X \times \Y)^t, 0 \leq t \leq T, S \text{ is } \C\text{-realizable} \},\] 
and also $\H(\C,\alg) = \bigcup_{T} \H(\C,\alg,T)$. 

We will also discuss online learning algorithms that produce 
\emph{randomized} predictors.  In this case, the formal 
definition is a function 
$\alg : (\X \times \Y)^* \times \X \to [0,1]$, 
with the interpretation that 
$\alg(x_{1:(t-1)},y_{1:(t-1)},x_t)$ is the \emph{probability} 
of predicting $1$ at time $t$.
Thus, 
\[\left| \alg(x_{1:(t-1)},y_{1:(t-1)},x_t) - y_t \right|\]
represents the \emph{probability} of a mistake at time $t$.
Clearly, a deterministic predictor is just the 
special case where $\alg(x_{1:(t-1)},y_{1:(t-1)},x_t) \in \{0,1\}$ always.
As above, $\bar{h}_t = \alg(x_{1:(t-1)},y_{1:(t-1)})$ 
denotes the function $\X \to [0,1]$ such that 
$\bar{h}_t(x) = \alg(x_{1:(t-1)},y_{1:(t-1)},x)$,
the hypothesis class of $\A$ is  
\[\H(\C,\alg,T) = \{ \alg(S) : S \in (\X \times \Y)^t, 0 \leq t \leq T, S \text{ is } \C\text{-realizable} \},\] 
and $\H(\C,\alg) = \bigcup_{T} \H(\C,\alg,T)$.
Also define $\H(\alg,T) = \H(\Y^{\X},\alg,T)$.

For any concept class $\C$, learning algorithm $\alg$, 
and $T \in \nats$, define the algorithm's \emph{mistake bound}
\begin{equation*} 
\MB(\C,\!\alg,\!T) \!=\! \max\!\left\{ \sum_{t=1}^{T} \ind[ \alg(x_{1:(t-1)},y_{1:(t-1)},x_t) \neq y_t ] : (x_1,y_1),\ldots,\!(x_T,y_T) \text{ is } \C\text{-realizable} \right\}\!.
\end{equation*}
Also define $\MB(\C,\alg) = \sup_{T} \MB(\C,\alg,T)$.
For any hypothesis class $\H$, define the 
\emph{optimal mistake bound 
for learning $\C$ with $\H$}: 
$\MB(\C,\H,T) = \min\{ \MB(\C,\alg,T) : \H(\C,\alg,T) \subseteq \H \}$ 
and $\MB(\C,\H) = \sup_{T} \MB(\C,\H,T)$.

\paragraph{Complexity Measures.} 
We express our results in terms of well-known combinatorial 
complexity measures from the learning theory literature.
The main quantity appearing in most of our results is the 
\emph{Littlestone dimension} \citep*{littlestone:88}, 
denoted $\LD(\C)$, defined as the largest $n \in \nats \cup \{0\}$ 
for which $\exists \{ x_{\mathbf{y}} : \mathbf{y} \in \Y^{t}, t \in \{0,\ldots,n-1\} \} \subseteq \X$ 
(interpreting $\Y^0 = \{()\}$) 
with the property that $\forall y_1,\ldots,y_n \in \Y$, 
$\exists h \in \C$ with $(h(x_{()}),h(x_{y_1}),h(x_{y_{1:2}}),\ldots,h(x_{y_{1:(n-1)}})) = (y_1,\ldots,y_n)$. 
If no such largest $n$ exists, define $\LD(\C) = \infty$.
Also define $\LD(\emptyset) = -1$.  
When $\LD(\C) < \infty$, 
one can show that $\LD$ can equivalently be defined 
inductively as $\max_x \min_y \LD(\C_{(x,y)})+1$, 
with $\LD(\emptyset)=-1$ as the base case.

Another important quantity appearing in our results is the 
\emph{VC dimension} \citep*{vapnik:71,vapnik:74}.
For any set $\Z$ and any non-empty set $\F$ of 
functions $\Z \to \{0,1\}$, 
the VC dimension $\vc(\F)$ is defined as 
the largest $n \in \nats \cup \{0\}$ 
for which $\exists z_1,\ldots,z_n \in \Z$ with 
$\{ (f(z_1),\ldots,f(z_n)) : f \in \F \} = \{0,1\}^n$: 
that is, every one of the $2^n$ possible binary patterns of length $n$ 
can be realized by evaluating some function in $\F$ on the sequence 
$z_1,\ldots,z_n$.  Define $\vc(\F) = \infty$ if no such largest $n$ exists.
In particular, we will be interested both in $\vc(\C)$, the 
VC dimension of the concept class, and also in the 
\emph{dual VC dimension}, $\dvc(\C)$, 
defined as the largest $n \in \nats \cup \{0\}$ 
for which $\exists h_1,\ldots,h_n \in \C$ with 
$\{ (h_1(x),\ldots,h_n(x)) : x \in \X \} = \{0,1\}^n$, 
or $\dvc(\C) = \infty$ if no such largest $n$ exists.

A final complexity measure appearing in our results 
is the \emph{dual Helly number} \citep*{bousquet:20}.
For any non-empty sets $\C$ and $\H$ of concepts, 
define the \emph{dual Helly number}, 
denoted $\helly(\C,\H)$, as the minimum 
$k \geq 2$ such that, for any 
$\H$-unrealizable set $S \subseteq (\X \times \Y)$, 
there exists a $\C$-unrealizable 
$S' \subseteq S$ with $|S'| \leq k$.
If no such $k$ exists, define $\helly(\C,\H)=\infty$.
The dual Helly number was used to characterize the 
sample complexity of proper PAC learning 
by \citet*{bousquet:20}.  However, it also 
previously appeared in the literature on 
learning from equivalence queries, 
where it is known as the 
\emph{strong consistency dimension}
\citep*{balcazar:02b,chase:20}.\footnote{Technically,
the strong consistency dimension requires the 
unrealizable set to be a partial function. 
The two definitions only differ in the case the 
strong consistency dimension is $1$, as the 
dual Helly number for $|\C|>1$ 
is never smaller than $2$ 
due to sets of the type $\{(x,0),(x,1)\}$.}
It has also appeared in the literature on 
distributed learning under the name 
\emph{co-VC dimension} \citep*{kane19a}.

When $\C$ is clear from the context, 
we omit the argument $\C$, writing $\vc$, $\dvc$, and $\LD$ 
for $\vc(\C),\dvc(\C)$, and $\LD(\C)$, respectively, 
and when $\H$ is also clear from the context, 
we write $\helly$ for $\helly(\C,\H)$.
To avoid trivial cases, 
we always suppose $\vc \geq 1$ and $\dvc \geq 1$ in all results. 

\paragraph{The SOA.}
We will make use of an online learning algorithm 
originally introduced by \citet*{littlestone:88}, 
known as the \emph{standard optimal algorithm}, 
denoted $\SOA$.
Specifically, for any hypothesis class $\H$, 
define a deterministic predictor 
$\SOA_{\H} : \X \to \Y$, which, for every 
$x \in \X$, predicts 
$\SOA_{\H}(x) = \argmax_{y \in \Y} \LD(\H_{(x,y)})$.
In particular, \citet*{littlestone:88} 
proved that if $\LD(\H) < \infty$, 
then every $x \in \X$ has at least one $y \in \Y$ 
with $\LD(\H_{(x,y)}) < \LD(\H)$, 
and noted that this immediately implies 
that the online learning algorithm $\alg_{\C,\SOA}$ 
which predicts  
$\SOA_{\C_{\{(x_i,y_i)\}_{i=1}^{t-1}}}(x_t)$, 
for each $t \leq T$, has mistake bound 
$\LD(\C)$.

The result of \citet*{littlestone:88} implies 
$\MB(\C,\Y^\X) \leq \LD(\C)$.  However, 
at present there is no known simple description of the 
hypothesis class $\H(\C,\alg_{\C,\SOA})$ of the $\SOA$ predictor.
One of the main contributions of the present work is 
to argue that there exists a learning algorithm $\alg$ 
with a \emph{simple} hypothesis class $\H(\C,\alg)$ 
which still achieves $\MB(\C,\alg) = O(\LD(\C))$.
In particular, we find such an algorithm with $\H(\C,\alg)$ 
contained in the set of \emph{majority votes} of 
$O(\dvc)$ elements of $\C$.  For instance, this has an 
important implication that the Littlestone and VC dimensions of 
$\H(\C,\alg)$ can be bounded in terms of $\LD(\C)$ and $\vc(\C)$: namely,
\[ \vc(\H(\C,\alg)) = O( \vc \dvc \log(\dvc) ) = O\!\left( \vc^2 2^{\vc} \right) \]
and
\[  \LD(\H(\C,\alg)) = O( \LD \dvc \log(\dvc) ) = O\!\left(\LD \vc 2^{\vc}\right).\]
These follow from composition theorems for the 
VC dimension \citep*[see Theorem 4.5 of][]{vidyasagar:03} 
and Littlestone dimension 
\citep*{alon:20:closure,ghazi:20:closure}, 
together with the relation $\dvc < 2^{\vc+1}$ due to  
\citet*{assouad:83}.

\subsection{Summary of Main Results for Online Learning}
\label{sec:results}

We briefly summarize the main results of this work.
Their detailed statements, and proofs, are presented in the sections below.

\begin{theorem}
\label{thm:MB}
Every pair of classes $\C,\H\subseteq\Y^{\X}$ satisfy $\MB(\C,\H) = O(\LD \helly \log(\helly))$.
\end{theorem}

\begin{remark}
\label{rem:lower-bound}
\citep*{angluin:87,littlestone:88,balcazar:02b}
It is known that we always have 
$\MB(\C,\H) \geq \max\{\LD, \helly-1\}$, 
and that there exist classes with 
$\MB(\C,\H) = \Omega( \LD \helly )$ 
and other classes with 
$\MB(\C,\H) = O(\max\{\LD,\helly\})$.
\end{remark}

For any classifiers $h_1,\ldots,h_k$, 
define 
\[\Maj(h_1,\ldots,h_k)(x) = \ind\!\left[ \sum_{i \leq k} h_i(x) \geq \sum_{i \leq k} (1-h_i(x)) \right].\] 
Define a hypothesis class 
$\Maj(\C^k) = \{ \Maj(h_1,\ldots,h_{k'}) : 1 \leq k' \leq k, h_1,\ldots,h_{k'} \in \C \}$.

\begin{theorem}
\label{thm:majorities}
Every class $\C$ satisfies $\MB(\C,\Maj(\C^{c\dvc})) = O(\LD)$, 
where $c$ is a finite universal constant.
\end{theorem}

In fact, we prove a stronger result with 
\emph{margins}.  This can also be interpreted 
as a result about \emph{randomized} predictors 
based on a distribution over $\C$.  
For any $k \in \nats$ define 
\[\Vote(h_1,\ldots,h_k)(x) = \frac{1}{k} \sum_{i \leq k} h_i(x),\] 
and $\Vote(\C^k) = \{ \Vote(h_1,\ldots,h_{k'}) : 1 \leq k' \leq k, h_1,\ldots,h_{k'} \in \C \}$.

\begin{theorem}
\label{thm:randomized-proper}
For any $\eps \in (0,1/2)$,
there is an algorithm $\alg$ 
with $\H(\C,\alg) \subseteq \Vote(\C^{c\dvc})$ 
(for a universal constant $c$) 
such that, for any $T \in \nats$, 
running $\alg$ on any $\C$-realizable sequence $\{(x_1,y_1)\}_{t=1}^T$, 
there are at most 
$O\!\left( \frac{\LD}{\eps}\log\frac{1}{\eps} \right)$
times $t \leq T$ where $| \bar{h}_t(x_t) - y_t | > \eps$.
\end{theorem}

We also prove a result for \emph{agnostic} online 
learning.  In this case, rather than bounding the 
number of mistakes, we are interested in the 
\emph{difference} of the number of mistakes made 
by $\alg$ and the \emph{minimum} number of mistakes 
made by any single $h \in \C$.  This is known as 
the \emph{regret} of the algorithm $\alg$.  
It is known that any algorithm 
whose regret approaches zero as $T\to\infty$, 
\emph{must} be capable of using randomized predictors.  
We establish the following result.

\begin{theorem}
\label{thm:agnostic}
For any $T \in \nats$, 
there is an algorithm $\alg$ 
with $\H(\alg,T) \subseteq \Vote(\C^{m})$, 
where $m = O\!\left(\frac{\dvc T}{\LD \log(T/\LD)}\right)$, 
such that for any sequence 
$(x_1,y_1),\ldots,(x_T,y_T) \in \X \times \Y$, 
\begin{equation*} 
\sum_{t=1}^{T} \left| \bar{h}_t(x_t) - y_t \right| - \min_{h \in \C} \sum_{t=1}^{T} \ind[ h(x_t) \neq y_t ] = O\!\left( \sqrt{\LD T \log(T/\LD)} \right).
\end{equation*}
\end{theorem}

Note that $\alg$ induces a randomized algorithm that on each round $t$ 
interprets the hypothesis in $\Vote(\C^m)$ which $\alg$ uses 
as a probability distribution $\pi_t$ on $\C$ having 
support size at most $m$, and predicts $h_t(x_t)$ 
for a randomly drawn $h_t \sim  \pi_t$.
Then, the above theorem implies that the expected regret of this
randomized algorithm is at most 
$O\!\left( \sqrt{ \LD T \log(T/\LD) } \right)$.
This strengthen a similar result by \citet*{rakhlin2015online} 
who gave a non-constructive proof under further restrictions on the class $\C$. 
It also matches the bounds by \citet*{ben2009agnostic,rakhlin2012relax}
which were achieved by improper algorithms.
Further, the above bound is tight up to the log factor,
as follows by the recent work by \citet*{alon:21}
who used the non-constructive framework of \citet*{rakhlin2015online}
to get an optimal bound. It remains open to prove the optimal bound constructively.

\subsection{When Does the Minimax Theorem Hold for VC Games?}
\label{sec:minimax-summary}

We define a general binary-valued zero-sum game as 
follows.
Let $\A$ and $\B$ be nonempty sets,
called the \emph{action sets}.
We suppose they are each 
equipped with a $\sigma$-algebra 
defining the measurable subsets; 
in particular, we suppose all singleton sets 
$\{a\} \subseteq \A$, $\{b\} \subseteq \B$,
are measurable.
Let $\val : \A \times \B \to \{0,1\}$ 
be a \emph{value function}, assumed to 
be measurable in the product $\sigma$-algebra.
For each $a \in \A$, we can interpret 
$\val(a,\cdot)$ as a function $\B \to \{0,1\}$.
We call $(\A,\B,\val)$ a \emph{VC game} 
if the VC dimension of 
$\{ \val(a,\cdot) : a \in \A \}$ is finite.\footnote{We note that it 
follows from a known relation 
of \citet*{assouad:83} 
between VC dimension and dual VC dimension 
that the VC dimension of  
$\{ \val(a,\cdot) : a \in \A \}$ is finite 
if and only if the VC dimension of 
$\{ \val(\cdot,b) : b \in \B \}$ is also finite.}
A \emph{subgame} of $(\A,\B,\val)$ is 
any game $(\A',\B',\val)$ where 
$\A' \subseteq \A$ and $\B' \subseteq \B$ 
are nonempty measurable subsets, 
and $\val$ here is interpreted as 
the restriction of 
$\val$ to $\A' \times \B'$.
For any countable sequences 
$\{a_i\}_{i \in \nats}$ 
in $\A$ and $\{b_i\}_{i \in \nats}$ in $\B$, 
we say $(\{a_i : i \in \nats\}, \{b_i : i \in \nats\}, \val)$ is an 
infinite triangular subgame 
if $\forall i,j \in \nats$, $\val(a_i,b_j) = \ind[ i \leq j ]$.
Generally, for a set $\S$ (equipped with a 
$\sigma$-algebra defining the measurable subsets), 
denote by $\Pi(\S)$ the set of all 
probability measures on $\S$.

\begin{theorem}
\label{thm:minimax}
A binary-valued VC game $(\A,\B,\val)$ 
satisfies 
\begin{equation*}
\inf_{P_A \in \Pi(\A')} \sup_{P_B \in \Pi(\B')} \E_{(a,b) \sim P_A \times P_B}[\val(a,b)] 
=
\sup_{P_B \in \Pi(\B')} \inf_{P_A \in \Pi(\A')} \E_{(a,b) \sim P_A \times P_B}[\val(a,b)] 
\end{equation*}
for all subgames $(\A',\B',\val)$
if and only if 
it has no 
infinite triangular subgame.
Moreover, this remains true even if 
$\Pi(\A')$, $\Pi(\B')$ are restricted to 
be just the probability measures having finite support.
\end{theorem}

We note that, in general, one cannot strengthen this result
    by replacing the $\inf$ and $\sup$ by $\min$ and $\max$.
    In other words, there are games for which the above 
    result applies, but where optimal maximin and minimax 
    strategies do not exist, as the 
    optimal value is witnessed only in the limit.
    One such simple example is the game {\it``Guess My Number''}.
    This game is played between two players whom we call Alice and Bob.
    Each of Alice and Bob privately picks a natural number, and Bob's goal
    is to pick the same number Alice picked.
    Bob wins the game if and only if they picked the same number.
    It is easy to see that Alice can win this game with probability arbitrarily close to $1$:
    indeed, if Alice picks a uniform distribution over $\{1,\ldots, N\}$,
    then she wins with probability at least $1-1/N$.
    However, since every distribution over $\nats$ must give a positive measure to at least one number, 
    there is no single strategy for Alice with which she wins with probability $1$.

\smallskip

\paragraph{Littlestone games.} As a special case of particular importance in 
this work, we say $(\A,\B,\val)$ is a 
\emph{Littlestone game} if 
$\{ \val(a,\cdot) : a \in \A \}$ 
has finite Littlestone dimension.
Due to a well-known connection between 
the Littlestone dimension and the 
so-called \emph{threshold dimension} 
it is clear that any Littlestone game 
has no infinite triangular subgame.
(\cite{shelah:78,hodges:97}, see also \citet*{AlonLMM19})
Thus, the following corollary is 
immediately entailed by Theorem~\ref{thm:minimax}.

\begin{corollary}
\label{cor:littlestone-games}
Any binary-valued Littlestone game
$(\A,\B,\val)$ satisfies
\begin{equation*}
\inf_{P_A \in \Pi(\A)} \sup_{P_B \in \Pi(\B)} \E_{(a,b) \sim P_A \times P_B}[\val(a,b)] 
=
\sup_{P_B \in \Pi(\B)} \inf_{P_A \in \Pi(\A)} \E_{(a,b) \sim P_A \times P_B}[\val(a,b)].
\end{equation*}
This remains true even if $\Pi(\A)$ and $\Pi(\B)$ are 
restricted to just the probability measures having finite support.
\end{corollary}

\subsection{Expression of the Results in Terms of Equivalence Queries}
\label{sec:EQ}

There is a well-known correspondence between the 
online learning setting of \citet*{littlestone:88}
and the setting of Exact learning from Equivalence Queries 
introduced by \citet*{angluin:87}.
In the problem of learning $\C$ 
using Equivalence Queries for $\H \supseteq \C$, 
there is some unknown target concept $\target \in \C$, 
and proceeding in rounds, 
on each round an algorithm proposes a hypothesis 
$h \in \H$ as a query to an oracle, 
which then either certifies that $h = \target$ 
or else returns a \emph{counterexample} 
$x \in \X$ such that $h(x) \neq \target(x)$.
The \emph{query complexity} $\QC_{{\rm EQ}}(\C,\H)$ 
is defined as the minimum number $q$   
such that there is an algorithm that, for any $\target \in \C$, 
regardless of the oracle's responses 
(as long as they are valid), 
the algorithm is guaranteed to query the oracle 
with $h=\target$ 
within at most $q$ queries; if no such number $q$ 
exists, $\QC_{{\rm EQ}}(\C,\H)$ is defined to be 
infinite.

It is easy to observe that 
$\QC_{{\rm EQ}}(\C,\H) = \MB(\C,\H)+1$.
To see this, note that 
we can always use an online learning algorithm 
to propose the hypotheses $h \in \H$, and update the 
learner using $(x,1-h(x))$ for the returned point $x$ 
on rounds where the oracle returns an $x$.  Since 
each such $x$ is a mistake for the online learner, 
an optimal learner will need at most $\MB(\C,\H)$ 
such rounds before its next hypothesis $h$ equals $\target$, 
in which case one final query suffices for the oracle 
to certify that $h=\target$.
In the other direction, given any optimal algorithm for 
Exact learning $\C$ with Equivalence Queries for $\H$, 
we can use the proposed hypothesis $h$ as an online 
learner's predictor until the first time when it makes 
a mistake $(x_t,y_t)$; that $x_t$ would be a valid 
response from the oracle, so we may feed this into the 
Exact learning algorithm, which then produces its next 
hypothesis $h$, which becomes the new hypothesis for 
the online learner for its prediction on the point $x_{t+1}$, 
and so on until the next mistake.  If the sequence 
$(x_1,y_1),(x_2,y_2),\ldots$ is $\C$-realizable, 
then this can continues 
for at most $\QC_{{\rm EQ}}(\C,\H)-1$ rounds 
before the algorithm produces a hypothesis $h$ 
that never makes another mistake on the sequence
(it need not be equal the target concept 
$\target$ if the number of mistakes 
is strictly smaller than $\QC_{{\rm EQ}}(\C,\H)-1$, 
but merely never encounters another counterexample 
to update on).

The problem of characterizing $\QC_{{\rm EQ}}(\C,\H)$ 
is a classic question in the learning theory literature
\citep*[e.g.,][]{angluin:87,angluin:90,hellerstein:96,balcazar:02b,chase:20}.
For the case of finite $\C$, it was shown by 
\citet*{balcazar:02b} that 
$\helly \leq \QC_{{\rm EQ}}(\C,\H) \leq \left\lceil \helly \ln(|\C|) \right\rceil$.
The lower bound of \citet*{littlestone:88} for $\MB(\C,\Y^\X)$ 
immediately implies a lower bound in terms of the Littlestone 
dimension $\LD$: namely, 
$\LD+1 \leq \QC_{{\rm EQ}}(\C,\H)$.
Recently, \citet*{chase:20} established an upper 
bound expressed in terms of the Littlestone dimension $\LD$, 
which therefore also holds for infinite classes: 
namely, $\QC_{{\rm EQ}}(\C,\H) \leq \helly^{\LD}$.
They in fact show a bound where $\helly$ is replaced 
by the sometimes-smaller \emph{consistency dimension}; 
we refer the reader to that work for the details.
This was the best known bound holding for all classes 
of a given $\LD$, with no dependence on $|\C|$.
Thus, our Theorem~\ref{thm:MB} immediately implies 
a new near-optimal bound: 
$\QC_{{\rm EQ}}(\C,H) = O( \LD \helly \log(\helly) )$.
This resolves the optimal query complexity, 
up to a factor $\log(\helly)$ and 
unavoidable gaps, 
a problem which has been studied for several decades, 
and moreover this solution 
represents an exponential improvement in 
the dependence on $\LD$ compared to the previous 
result of \citet*{chase:20}.

\section{The Optimal Mistake Bound for Learning $\C$ with $\H$}
\label{sec:MB-C-H}

This section presents the details of the algorithm 
and proof of Theorem~\ref{thm:MB}.

For a finite set $\Hyps$ of pairs 
$(C_i,w_i)$, where $C_i \subseteq \C$ 
and $w_i \geq 0$, 
define 
\begin{equation*} 
\Vote(\Hyps)(x) = \frac{\sum_{(C_i,w_i) \in \Hyps} w_i \SOA_{C_i}(x)}{\sum_{(C_j,w_j) \in \Hyps} w_j},
\end{equation*}
and 
$\Maj(\Hyps)(x) = \ind[ \Vote(\Hyps)(x) \geq 1/2 ]$.
Also, for $\epsilon \in [0,1]$, define 
\begin{equation*}
\HighVote(\Hyps,\epsilon) = \{ (x,\Maj(\Hyps)(x)) : x \in \X, \Vote(\Hyps)(x)\in [0,\epsilon] \cup [1-\epsilon,1] \}.
\end{equation*}
Consider the following online learning 
algorithm, for any $\C$ and $\H$ 
with $\helly = \helly(\C,\H) < \infty$,
executed on any $\C$-realizable sequence 
$(x_1,y_1),\ldots,(x_T,y_T)$ in $\X \times \Y$.

\begin{bigboxit}
0. Initialize $\Hyps = \{ (C_1,w_1) \} = \{ (\C,1) \}$, $\eta = 1/(2\helly)$, $t=1$\\
1. Repeat while $t \leq T$\\
2.\quad If $\HighVote(\Hyps,\eta)$ is $\H$-realizable\\
3.\qquad Choose $\hat{h}_t \in \H$ correct on $\HighVote(\Hyps,\eta)$, predict $\hat{y}_t = \hat{h}_t(x_t)$\\
4.\qquad If $\hat{y}_t \neq y_t$ (i.e., mistake)\\
5.\qquad\quad For each $(C_i,w_i) \in \Hyps$\\ 
6.\qquad\qquad If $\SOA_{C_i}(x_t) \neq y_t$, set $w_i \gets \eta \cdot w_i$\\
7.\qquad\qquad $C_i \gets \{ h \in C_i : h(x_t) = y_t \}$\\
8.\qquad\qquad If $C_i = \emptyset$, remove $(C_i,w_i)$ from $\Hyps$\\
9.\qquad $t \gets t+1$\\
10.\quad Else let $\{(\tx_1,\ty_1),\ldots,(\tx_{\helly},\ty_{\helly})\} \subseteq \HighVote(\Hyps,\eta)$ be $\C$-unrealizable\\
11.\qquad For each $(C_i,w_i) \in \Hyps$\\
12.\qquad\quad For each $j \leq \helly$\\
13.\qquad\qquad If $\SOA_{C_i}(\tx_j) = \ty_j$, $w_{ij} \gets \eta \cdot w_i$, else $w_{ij} \gets w_i$\\
14.\qquad\qquad Let $C_{ij} = \{ h \in C_i : h(\tx_j) = 1-\ty_j \}$\\
15.\qquad $\Hyps \gets \{ (C_{ij},w_{ij}) : C_{ij} \neq \emptyset \}$
\end{bigboxit}

\begin{theorem}
\label{thm:upper}
$\MB(\C,\H) \leq 4 \LD \helly \ln(2\helly)$, 
achieved by the above algorithm.
\end{theorem}
\begin{proof}
Suppose $\LD$ and $\helly$ are both finite.
Let $\target \in \C$ be a concept  
correct on $\{(x_t,y_t)\}_{t=1}^{T}$.
On each round where $\HighVote(\Hyps,\eta)$ 
is $\H$-realizable and $\hat{y}_t \neq y_t$, 
clearly $\target$ is 
correct on the $(x_t,y_t)$ used to update the 
sets $C_i$.
Moreover, on each round where $\HighVote(\Hyps,\eta)$ 
is not $\H$-realizable, the sequence $\{(\tx_j,\ty_j)\}_{j=1}^{\helly}$ is not 
$\C$-realizable, and therefore 
there is at least one $\tx_j$ for which 
$1-\ty_j = 1-{\rm Maj}(\Hyps)(\tx_j) = \target(\tx_j)$.
So if we think of the set $\Hyps$ as developing like a tree 
(with each element at the end of the round 
being updated from 
the previous round using the point $(x_t,y_t)$ 
on rounds of the first type, 
or else branching off of a previous element  
by updating with some 
$(\tx_j,1-\ty_j)$ on rounds of the second type), 
then there is a path in the tree where all of the 
updates are for $\target$-consistent examples.
To put this more formally, if at the beginning of 
round $n$ there exists $(C,w) \in \Hyps$ with 
$\target \in C$, then (by the above observations) 
at the end of round $n$ 
there will still exist some $(C',w') \in \Hyps$ with 
$\target \in C'$, so that by induction we maintain 
this property for all rounds $n$.
Moreover, we note that, at the end of any round $n$, 
any $(C,w) \in \Hyps$ having $\target \in C$ 
must have $w \geq \eta^{\LD}$, since on any 
sequence of labeled examples $(x'_i,\target(x'_i))$,
there are at most $\LD$ times $i$ with 
$\SOA_{\C_{\{(x'_j,\target(x'_j)) : j < i\}}}(x'_i) \neq \target(x'_i)$, 
as established by \citet*{littlestone:88}.
Thus, after $n$ rounds of the outermost loop,  
$\exists (C^*_n,w^*_n) \in \Hyps$ with 
\begin{equation*} 
w^*_n \geq \eta^{\LD} = (1/(2\helly))^{\LD}.
\end{equation*}

Now suppose the algorithm executes the outermost loop 
at least $n$ times, and consider the total state of the algorithm 
after completing round $n$ of the outermost loop.
Let $t_n$ denote the value of $t$ after 
completing this round, and let $M_n$ 
denote the number of $t \in \{1,\ldots,t_n-1\}$ 
with $\hat{y}_t \neq y_t$: that is, the 
number of mistakes on the actual data sequence 
within the first $n$ rounds.
Let $N_n$ denote the number of the first $n$ rounds where 
$\HighVote(\Hyps,\eta)$ is not $\H$-realizable.
Let $W_n$ denote the total weight in $\Hyps$ 
after completing $n$ rounds.

On any round $n' \leq n$ 
where $\HighVote(\Hyps,\eta)$ is 
$\H$-realizable and $\hat{y}_{t} \neq y_{t}$  (for $t = t_{n'}-1$),
since every $(x,y) \in \HighVote(\Hyps,\eta)$ has 
$\hat{h}_t(x) = y$, yet we know that 
$\hat{h}_t(x_{t}) = \hat{y}_{t} \neq y_{t}$, 
it must be that 
$(x_{t},y_{t}) \notin \HighVote(\Hyps,\eta)$; 
therefore 
at least $\eta$ fraction of the total weight 
is multiplied by $\eta$ in Step 6: 
that is, 
$W_{n'} \leq \eta^2 W_{n'-1} + (1-\eta) W_{n'-1} 
= \left( 1 - \eta (1-\eta) \right) W_{n'-1}
= \left( 1 - (2\helly-1)/(2\helly)^2 \right) W_{n'-1}$.
On the other hand, on any round $n' \leq n$ 
where $\HighVote(\Hyps,\eta)$ is \emph{not} 
$\H$-realizable, each $j \leq \helly$ 
has $(\tx_j,\ty_j)$ in $\HighVote(\Hyps,\eta)$, 
so that at least $1-\eta$ fraction of the total 
of weights $w_i$ have $w_{ij}  = \eta \cdot w_{i}$; 
thus, $W_{n'} \leq \helly ( \eta (1-\eta) W_{n'-1} + \eta W_{n'-1} ) 
= \left( 1 - 1/(4\helly) \right) W_{n'-1}$.
By induction, we have 
\begin{align*}
W_{n} & \leq 
(1-(2\helly-1)/(2\helly)^2)^{M_n}(1-1/(4\helly))^{N_n}
\\ & < \exp\{-M_n (2\helly-1)/(2\helly)^2\} \exp\{-N_n / (4\helly)\}
\leq  \exp\{-M_n / (4\helly)\} \exp\{-N_n / (4\helly)\}
\end{align*}
since $\helly \geq 1$.
Thus, since $w^*_n \leq W_n$, 
we have 
\begin{equation*}
M_n / (4\helly) + N_n / (4\helly) < \LD \ln(2\helly).
\end{equation*} 
This has two important implications.
First, since we always have 
$n = N_n + t_n-1$, 
and $t_n \leq T+1$ while the above inequality 
implies $N_n < 4 \helly \LD \ln(2\helly)$, 
we have that the algorithm will terminate after 
a finite number of rounds.
Second, the above inequality further implies 
that $M_n < 4 \helly \LD \ln(2\helly)$ 
for all rounds in the algorithm, so that 
this is also a bound on the total number of 
mistakes at the time of termination,
after predicting for all $T$ points in the 
sequence. 
This completes the proof.
\end{proof}

\section{A Characterization of Games Satisfying the Minimax Theorem for All Subgames}
\label{sec:minimax}

A key component of the proofs of our results 
on learning with majority votes and randomized 
proper predictors (Theorems~\ref{thm:majorities}, \ref{thm:randomized-proper}, \ref{thm:agnostic}) 
is a general characterization of games 
for which the minimax theorem holds: 
namely, Theorem~\ref{thm:minimax}
stated Section~\ref{sec:results}.
We present the proof here, separating the 
two parts of the claim.

\begin{proposition}
\label{prop:minimax-sufficient}
Any binary-valued VC game $(\A,\B,\val)$ with no 
infinite triangular subgame satisfies 
\begin{equation*}
\inf_{P_A \in \Pi(\A)} \sup_{P_B \in \Pi(\B)} \E_{(a,b) \sim P_A \times P_B}[\val(a,b)] 
=
\sup_{P_B \in \Pi(\B)} \inf_{P_A \in \Pi(\A)} \E_{(a,b) \sim P_A \times P_B}[\val(a,b)].
\end{equation*}
Moreover, this remains true even if 
$\Pi(\A)$, $\Pi(\B)$ are restricted 
to be just the probability measures having finite support.
\end{proposition}
\begin{proof}
We prove this result in the contrapositive.
Suppose $(\A,\B,\val)$ is a binary-valued 
VC game, let 
\begin{equation*}
\alpha = \sup_{P_B \in \Pi(\B)} \inf_{P_A \in \Pi(\A)} \E_{(a,b) \sim P_A \times P_B}[\val(a,b)]
\end{equation*}
and let 
\begin{equation*}
\beta = \inf_{P_A \in \Pi(\A)} \sup_{P_B \in \Pi(\B)} \E_{(a,b) \sim P_A \times P_B}[\val(a,b)],
\end{equation*}
and let us suppose $\alpha < \beta$.
Let $\vc_{A}$ denote the VC dimension of 
$\{ \val(a,\cdot) : a \in \A \}$ 
and $\vc_{B}$ denote the VC dimension of 
$\{ \val(\cdot,b) : b \in \B \}$.
Note that these are both finite by the 
assumption that $(\A,\B,\val)$ is VC game, 
and the fact that $\vc_{B} < 2^{\vc_{A}+1}$ 
\citep*{assouad:83}.

We inductively define two sequences of 
mixed strategies 
$\{ P_A^t \}_{t \in \nats}$,$\{ P_B^t \}_{t \in \nats}$, where each $P_A^t \in \Pi(\A)$ 
and each $P_B^t \in \Pi(\B)$ 
such that, for a numerical constant $c$, 
each $P_A^t$ is supported on at most 
$\frac{c \vc_B}{(\beta-\alpha)^2}$ elements 
of $\A$ 
and has 
\begin{equation}
\label{eqn:PA-requirement}
\sup_{P_B \in \Pi\!\left( \bigcup_{i < t} \supp(P_B^i) \right)} \E_{(a,b) \sim P_A^t \times P_B}[ \val(a,b) ] \leq \frac{2\alpha+\beta}{3}
\end{equation}
while each $P_B^t$ is supported on at most 
$\frac{c \vc_A}{(\beta-\alpha)^2}$ elements 
of $\B$ and has 
\begin{equation}
\label{eqn:PB-requirement}
\inf_{P_A \in \Pi\!\left( \bigcup_{i \leq t} \supp(P_A^i) \right)} \E_{(a,b) \sim P_A \times P_B^t}[ \val(a,b) ] \geq \frac{\alpha+2\beta}{3}.
\end{equation}

As a base case, 
let $P_A^1 = \ind_{\{a\}}$ for some $a \in \A$: that is, $P_A^1$ is any pure strategy.
Now fix any $t \in \nats$ and suppose 
there exist $P_A^i$, $i \in \{1,\ldots,t\}$ 
and $P_B^i$, $i \in \{1,\ldots,t-1\}$ 
satisfying the above properties.
To complete the inductive construction, 
it suffices to specify $P_A^t$ and $P_B^{t+1}$ 
to extend the sequences.
Letting $B_{<t} = \bigcup_{i < t} \supp(P_B^i)$,
since this is a finite set (by the inductive 
hypothesis), there is a finite 
number of distinct sequences 
$\{\val(a,b)\}_{b \in B_{<t}}$ realized 
by elements $a \in \A$.
Thus, there exists a finite 
set $\A' \subseteq \A$ such that 
every such sequence 
$\{\val(a,b)\}_{b \in B_{<t}}$ 
witnessed by an $a \in \A$ 
is also witnessed by 
$\{\val(a',b)\}_{b \in B_{<t}}$
for some $a' \in \A'$.
In particular, by the classic 
minimax theorem for finite games \citep*{von-neumann:44}, 
$\exists P_{A}^{*} \in \Pi(\A')$ 
such that 
\begin{align*}
&\sup_{P_{B} \in \Pi(B_{<t})} \E_{(a,b) \sim P_{A}^{*} \times P_{B}}[ \val(a,b) ] 
= \sup_{P_{B} \in \Pi(B_{<t})} \inf_{P_{A} \in \Pi(\A')} \E_{(a,b) \sim P_{A} \times P_{B}}[ \val(a,b) ]
\\ & = \sup_{P_{B} \in \Pi(B_{<t})} \inf_{P_{A} \in \Pi(\A)} \E_{(a,b) \sim P_{A} \times P_{B}}[ \val(a,b) ] \leq 
\sup_{P_{B} \in \Pi(\B)} \inf_{P_{A} \in \Pi(\A)} \E_{(a,b) \sim P_{A} \times P_{B}}[ \val(a,b) ]
= \alpha,
\end{align*}
where the second equality follows from the 
assumed property of $\A'$ and the 
subsequent inequality follows from the fact that 
every $P_{B} \in \Pi(B_{<t})$ 
is a restriction to $B_{<t}$ of some 
$P'_{B} \in \Pi(\B)$ with 
$\supp(P'_{B}) = \supp(P_{B})$.

Now since $B_{<t}$ is a finite set, 
the classic uniform convergence property 
of VC classes holds (Lemma~\ref{lem:vc-bound} of Appendix~\ref{sec:vc-bounds}), 
which in particular 
implies that there exists a sequence 
$\{a_{i}\}_{i \leq m}$ in $\A'$ for some  
$m \leq \frac{c \vc_{B}}{(\beta-\alpha)^2}$ 
(for a universal constant $c$)
such that, defining $P_{A}^{t}$ as 
the empirical measure, i.e.,  
$P_{A}^{t}(\cdot) = \frac{1}{m} \sum_{i=1}^{m} \ind[ a_{i} \in \cdot ]$, 
it holds that every 
$b \in B_{<t}$ satisfies 
$\E_{a \sim P_{A}^{t}}[\val(a,b)] 
\leq \E_{a \sim P_{A}^{*}}[\val(a,b)] + \frac{\beta-\alpha}{3}$.
In particular, this implies 
\begin{equation*}
\sup_{P_{B} \in \Pi(B_{<t})} \E_{(a,b) \sim P_{A}^{t} \times P_{B}}[ \val(a,b) ]
\leq \alpha + \frac{\beta-\alpha}{3} 
= \frac{2\alpha+\beta}{3}.
\end{equation*}
That is, \eqref{eqn:PA-requirement} holds.

Applying the same argument 
to the set 
$\A_{\leq t} = \bigcup_{i \leq t} \supp(P_A^i)$, 
implies the existence of 
a finite set $\B' \subseteq \B$ 
with 
$\sup_{P_{A} \in \Pi(A_{\leq t})} \E_{(a,b) \sim P_{A} \times P_{B}^{*}}[ \val(a,b) ] 
\geq \beta$, 
and a sequence $\{b_{i}\}_{i \leq m'}$ 
in $\B'$ for some 
$m' \leq \frac{c \vc_{A}}{(\beta-\alpha)^2}$ 
such that, defining $P_{B}^{t+1}$ as the 
empirical measure, $P_{B}^{t+1} = \frac{1}{m'} \sum_{i=1}^{m'} \ind[ b_{i} \in \cdot ]$,
it holds that every $a \in A_{\leq t}$ 
satisfies 
$\E_{b \sim P_{B}^{t+1}}[\val(a,b)] 
\geq \E_{b \sim P_{B}^{*}}[\val(a,b)] -  \frac{\beta-\alpha}{3}$, 
which implies 
\begin{equation*}
\inf_{P_{A} \in \Pi(A_{\leq t})} \E_{(a,b) \sim P_{A} \times P_{B}^{t+1}}[ \val(a,b) ]
\geq \beta - \frac{\beta-\alpha}{3} 
= \frac{\alpha+2\beta}{3},
\end{equation*}
so that \eqref{eqn:PB-requirement} holds.
By the principle of induction, we have 
established the claimed existence of the 
above infinite sequences $P_{A}^{t}$,$P_{B}^{t}$.

For each $i \in \nats$, 
let $m_{i}^{A} = |\supp(P_{A}^{i})|$ 
and $\{a_{i,1},\ldots,a_{i,m_{i}^{A}}\} = \supp(P_{A}^{i})$, 
and let 
$m_{i}^{B} = |\supp(P_{B}^{i})|$ 
and $\{b_{i,1},\ldots,b_{i,m_{i}^{B}}\} = \supp(P_{B}^{i})$, 
Now for each $i,j \in \nats$, 
define two matrices: 
$C^{ij}$ is a $m_{i}^{A} \times m_{j}^{B}$ 
matrix with entries 
$C^{ij}_{k\ell} = \val(a_{i,k},b_{j,\ell})$, 
and $D^{ij}$ is a $m_{j}^{A} \times m_{i}^{B}$ 
matrix with entries 
$D^{ij}_{k\ell} = \val(a_{j,k},b_{i,\ell})$.
Note that every $i,j \in \nats$ have 
$C^{ij}$ and $D^{ij}$ of sizes no larger
than $\frac{c \vc_{B}}{(\beta-\alpha)^2} \times \frac{c \vc_{A}}{(\beta-\alpha)^2}$.
In particular, this implies that there 
are only a finite number of possible 
$(C^{ij},D^{ij})$ pairs witnessed among 
choices of $i,j \in \nats$.
If we consider each possible $(C,D)$ pair as a 
\emph{color} for the pairs $(i,j) \in \nats^2$
with $i < j$,
the infinite Ramsey theorem implies 
that there exists an infinite increasing 
sequence $i_{1},i_{2},\ldots$ in $\nats$ 
such that $\forall (s,t),(s',t') \in \nats^2$ 
with $s < t$ and $s' < t'$, 
we have $(C^{i_{s} i_{t}},D^{i_{s} i_{t}}) 
= (C^{i_{s'} i_{t'}},D^{i_{s'} i_{t'}})$.
Let $(C^*,D^*)$ denote this common value 
for the pair $(C^{i_{s} i_{t}},D^{i_{s} i_{t}})$.

Now we claim $\exists k^{*},\ell^{*}$ with 
$C^{*}_{k^{*} \ell^{*}} = 1$ and $D^{*}_{k^{*} \ell^{*}} = 0$.
This is because, for each $(s,t) \in \nats$ 
with $s < t$, we have $i_{s} < i_{t}$, 
so that 
\begin{align*}
& \inf_{P_{A} \in \Pi(\supp(P_{A}^{i_{s}}))} 
\sup_{P_{B} \in \Pi(\supp(P_{B}^{i_{t}}))} \E_{(a,b) \sim P_{A} \times P_{B}}[ \val(a,b) ] 
\\ & \geq 
\inf_{P_{A} \in \Pi(A_{\leq i_{t}})} \E_{(a,b) \sim P_{A} \times P_{B}^{i_{t}}}[ \val(a,b) ] 
\geq \frac{\alpha+2\beta}{3},
\end{align*}
while
\begin{align*}
& \inf_{P_{A} \in \Pi(\supp(P_{A}^{i_{t}}))} 
\sup_{P_{B} \in \Pi(\supp(P_{B}^{i_{s}}))} 
\E_{(a,b) \sim P_{A} \times P_{B}}[ \val(a,b) ]
\\ & \leq 
\sup_{P_{B} \in \Pi(B_{< i_{t}})} \E_{(a,b) \sim P_{A}^{i_{t}} \times P_{B}}[ \val(a,b) ] 
\leq \frac{2\alpha+\beta}{3}.
\end{align*}
In other words, the value of the finite 
game represented by $C^{*}$ is strictly 
greater than the value of the finite 
game represented by $D^{*}$.
For this to be true,  
there must exist at least one 
pair $k^{*},\ell^{*}$ with $C^{*}_{k^{*}\ell^{*}} > D^{*}_{k^{*}\ell^{*}}$, 
which (since these are binary-valued games) 
implies $C^{*}_{k^{*}\ell^{*}}=1$ and $D^{*}_{k^{*}\ell^{*}}=0$.
Fix any such pair $k^{*},\ell^{*}$.

To complete the proof, we use this fact 
to construct an infinite triangular subgame.
Define $\tilde{a}_{t} = a_{i_{2t-1},k^{*}}$ 
and $\tilde{b}_{t} = b_{i_{2t},\ell^{*}}$ 
for all $t \in \nats$.
Note that, for any $s,t \in \nats$, 
if $s \leq t$, then $i_{2s-1} < i_{2t}$, 
so that 
$\val(\tilde{a}_{s},\tilde{b}_{t})
= \val(a_{i_{2s-1},k^{*}},b_{i_{2t},\ell^{*}})
= C^{i_{2s-1} i_{2t}}_{k^{*} \ell^{*}} 
= C^{*}_{k^{*} \ell^{*}} = 1$.
On the other hand, if $s > t$, 
then $i_{2s-1} > i_{2t}$, so that 
$\val(\tilde{a}_{s},\tilde{b}_{t})
= \val(a_{i_{2s-1},k^{*}},b_{i_{2t},\ell^{*}})
= D^{i_{2t} i_{2s-1}}_{k^{*} \ell^{*}} 
= D^{*}_{k^{*} \ell^{*}} = 0$.
Together we have 
$\val(\tilde{a}_{s},\tilde{b}_{t}) = \ind[ s \leq t ]$ for all $s,t \in \nats$, 
so that  
$(\{\tilde{a}_{t} : t \in \nats\}, \{\tilde{b}_{t} : t \in \nats\},\val)$
is an infinite triangular subgame.

The final claim about restricting $\Pi(\A)$ and $\Pi(\B)$ 
to have finite support follows by noting that all probability 
measures $P_A$ and $P_B$ used in the above proof have finite 
support, so that if the claimed equality is violated for
finite-support probability measures, then the above 
construction still implies the existence of an infinite 
triangular subgame.
\end{proof}

\begin{proposition}
\label{prop:minimax-necessary}
Any binary-valued VC game $(\A,\B,\val)$ with an 
infinite triangular subgame\linebreak $(\A',\B',\val)$ 
satisfies 
\begin{equation*}
1 = \inf_{P_A \in \Pi(\A')} \sup_{P_B \in \Pi(\B')} \E_{(a,b) \sim P_A \times P_B}[\val(a,b)] 
> 
\sup_{P_B \in \Pi(\B')} \inf_{P_A \in \Pi(\A')} \E_{(a,b) \sim P_A \times P_B}[\val(a,b)] = 0.
\end{equation*}
\end{proposition}
\begin{proof}
Let $\{a_{i} : i \in \nats\} = \A'$ 
and $\{b_{i} : i \in \nats\} = \B'$ 
so that $\val(a_{i},b_{j}) = \ind[ i \leq j ]$, 
as guaranteed by the defining property 
of an infinite triangular subgame.
Now note that, for any 
$P_{A} \in \Pi(\A')$, we have 
\begin{equation*}
\sup_{P_B \in \Pi(\B')} \E_{(a,b) \sim P_A \times P_B}[\val(a,b)] 
\geq 
\lim_{j \to \infty} \E_{a_{i} \sim P_A}[\val(a_{i},b_{j})] 
= \E_{a_{i} \sim P_A}\!\left[ \lim_{j \to \infty} \val(a_{i},b_{j}) \right] 
= 1,
\end{equation*}
where the first equality is due to the 
monotone convergence theorem, which 
applies since $\val(a_{i},b_{j})$ is 
nondecreasing in $j$, with limiting 
value $1$ (achieved for all $j \geq i$).
Since the values are bounded by $1$, 
this implies the leftmost claimed 
equality.

Likewise, for any $P_{B} \in \Pi(\B')$, 
we have 
\begin{equation*}
\inf_{P_A \in \Pi(\A')} \E_{(a,b) \sim P_A \times P_B}[\val(a,b)] 
\leq  
\lim_{i \to \infty} \E_{b_{j} \sim P_B}[\val(a_{i},b_{j})] 
= \E_{b_{j} \sim P_B}\!\left[ \lim_{i \to \infty} \val(a_{i},b_{j}) \right] 
= 0,
\end{equation*}
where the first equality  
holds by the dominated convergence theorem, 
since $|\val(a_{i},b_{j})| \leq 1$ and 
$\val(a_{i},b_{j})$ has limit $0$ as 
$i \to \infty$ (achieved for all $i > j$).
Since the values are bounded below by $0$, 
this implies the rightmost claimed 
equality, and this completes the proof.
\end{proof}

Theorem~\ref{thm:minimax} follows 
immediately from
Propositions~\ref{prop:minimax-sufficient} and \ref{prop:minimax-necessary}, as follows.

\begin{proof}[of Theorem~\ref{thm:minimax}]
If $(\A,\B,\val)$ is a VC game for which 
some subgame $(\A',\B',\val)$ satisfies 
\begin{equation}
\label{eqn:minimax-fail}
\inf_{P_A \in \Pi(\A')} \sup_{P_B \in \Pi(\B')} \E_{(a,b) \sim P_A \times P_B}[\val(a,b)] 
> 
\sup_{P_B \in \Pi(\B')} \inf_{P_A \in \Pi(\A')} \E_{(a,b) \sim P_A \times P_B}[\val(a,b)],
\end{equation}
then Proposition~\ref{prop:minimax-sufficient} 
implies $(\A',\B',\val)$ has an 
infinite triangular subgame, which is 
therefore also a subgame of $(\A,\B,\val)$.
By Proposition~\ref{prop:minimax-sufficient}, 
this remains true even if we restrict to finite-support 
probability measures.
For the other direction, any game $(\A,\B,\val)$ 
with an infinite triangular subgame 
$(\A',\B',\val)$ satisfies 
\eqref{eqn:minimax-fail} for these 
$\A'$, $\B'$ 
by Proposition~\ref{prop:minimax-necessary}.
This completes the proof.
\end{proof}

\section{Optimal Online Learning with Small Majorities}
\label{sec:majorities}

The question we address here is how simple of a 
hypothesis class $\H$ can we use while ensuring that 
an optimal mistake bound is still achievable 
(up to numerical constant factors).
Here we find this is possible using $\H$ based on 
majority votes of $O(\dvc)$ classifiers from $\C$.
We first prove a coarse bound achieved by 
the general algorithm above by  
showing that $\helly = O(\vc)$ for this class $\H$. 
This yields a mistake bound $O(\vc \LD \log(\vc))$.
We then refine this by a direct analysis, 
showing that an essentially-similar algorithm
achieves a mistake bound of $O(\LD)$ 
with this same class $\H$.

\subsection{Supporting Lemmas}
\label{sec:maj-prelims}

We first establish the following helpful lemmas.
The first is a simple application of the minimax theorem 
and the classic result on the size of $\eps$-approximating 
sets for VC classes.
The second is based on a technique for sparsifying majority 
votes having a margin, proposed 
by \citet*{moran:16}.

\begin{lemma}
\label{lem:eps-net}
Suppose $\LD < \infty$.
For any $\eps \in (0,1]$ and any set 
$S \subseteq \X \times \Y$, 
if, for every finite-support 
probability measure $\pi$ on $\C$, 
there exists $(x,y) \in S$ such that 
$\pi( h : h(x) \neq y ) > \eps$, 
then for every integer 
$m \geq \frac{\relconst \vc}{\eps} \log\!\left(\frac{1}{\eps}\right)$ 
(for a universal constant $\relconst$), 
there exists a sequence $S'$ in $S$ with $|S'|=m$
such that every $h \in \C$ has 
$\frac{1}{m} \sum_{(x,y) \in S'} \ind[ h(x) \neq y ] > \frac{\eps}{2}$.
\end{lemma}
\begin{proof}
By Corollary~\ref{cor:littlestone-games}, 
there exists a finite-support probability measure 
$P$ on $S$ such that 
every $h \in \C$ has $P( (x,y) : h(x) \neq y ) \geq \frac{2}{3}\eps$.
Since $P$ has finite support, the classic 
relative uniform convergence bounds of \citet*{vapnik:74}
hold (see Lemma~\ref{lem:vc-bound} in Appendix~\ref{sec:vc-bounds}).
This implies that for a universal constant $\relconst$,
for any $m \geq \frac{\relconst \vc}{\eps} \ln\!\left(\frac{1}{\eps}\right)$, 
there exists a sequence $S'$ of length $m$ 
in $S$ such that every $h \in \C$ satisfies 
$\frac{1}{m} \sum_{(x,y) \in S'} \ind[ h(x) \neq y ] >  \frac{\eps}{2}$.
\end{proof}

\begin{lemma}
\label{lem:dual-eps-approx}
For any $\eps \in (0,1/2)$ and any set $S \subseteq \X \times \Y$, 
if there exists a finite-support 
probability measure $\pi$ on $\C$ 
such that every $(x,y) \in S$ satisfies  
$\pi( h : h(x) \neq y ) \leq \eps$, 
then there exists a multiset $\C' \subseteq \C$ with
$|\C'| \leq \frac{\unifconst \dvc}{\eps^2}$ 
(for $\unifconst$ a universal constant) 
such that every $(x,y) \in S$ satisfies 
$\frac{1}{|\C'|} \sum_{h \in \C'} \ind[ h(x) \neq y ] < 2\eps$.
\end{lemma}
\begin{proof} 
To prove this, we apply the classic uniform convergence 
guarantees based on the chaining argument.  
Specifically, we apply 
Lemma~\ref{lem:vc-bound} of 
Appendix~\ref{sec:vc-bounds} 
with $\Z = \C$ and $\F$ the set of 
functions $g_x : \C \to \{0,1\}$, $x \in \X$,  
defined as $g_x(h) = h(x)$ for $h \in \C$.
The existence of $\C'$ with the claimed properties 
then follows immediately from Lemma~\ref{lem:vc-bound} 
by noting that $\vc(\F) = \dvc(\C)$.
\end{proof}

\subsection{The Dual Helly Number of Small Majorities}
\label{sec:helly-of-majorities}

To start, we state a coarse bound based on a direct 
application of Theorem~\ref{thm:upper}, by 
bounding the dual Helly number.

\begin{proposition}
\label{prop:helly-of-majorities}
For a numerical constant $c$ and $\H = \Maj(\C^{c\dvc})$, 
it holds that 
$\helly = O(\vc)$, 
and hence 
$\MB(\C,\H) = O(\LD \vc \log(\vc))$.
\end{proposition}

\begin{proof}
The statement is vacuous for $\LD = \infty$, so suppose $\LD < \infty$.
Let $c = 16 \unifconst$, for $\unifconst$ from Lemma~\ref{lem:dual-eps-approx}.
Let $S$ be a set not realizable by $\H$.
In particular, applying Lemma~\ref{lem:dual-eps-approx} 
with $\eps = 1/4$, we conclude that for every 
finite-support probability measure $\pi$ on $\C$, 
there exists $(x,y) \in S$ such that 
$\pi( h : h(x) \neq y ) > 1/4$.
Lemma~\ref{lem:eps-net} then implies there exists 
$S' \subseteq S$ with 
$|S'| \leq \left\lceil 4\relconst \vc \log(4) \right\rceil \leq 7\relconst \vc$, 
such that every $h \in \C$ has 
$\frac{1}{|S'|} \sum_{(x,y) \in S'} \ind[ h(x) \neq y ] > \frac{1}{8}$.
In particular, this implies $S'$ is not realizable by $\C$.
Thus, $\helly \leq 7\relconst \vc$.
A bound 
$\MB(\C,\H) \leq 28\relconst \LD \vc \ln(14\relconst \vc)$
then follows from Theorem~\ref{thm:upper}.
\end{proof}

\subsection{Optimal Mistake Bound via a Direct Analysis}
\label{sec:opt-majorities}

Proposition~\ref{prop:helly-of-majorities} reveals that 
it is possible to learn any $\C$ using majorities 
of $O(\dvc)$ elements of $\C$, with mistake bound at most 
$O(\LD \vc \log(\vc))$.
Here we find that this 
can be improved to an essentially \emph{optimal}
mistake bound: 
that is, $O(\LD)$.
To put this another way, we find that the 
optimal $\SOA$ predictor of 
\citet*{littlestone:88} can be approximated (in the appropriate sense) by small majority votes of classifiers from $\C$.

The following is a more-detailed form of 
Theorem~\ref{thm:majorities}, from which the 
original statement of Theorem~\ref{thm:majorities} 
immediately follows.

\begin{theorem}
\label{thm:majorities-detailed}
For $c = 36\unifconst$ 
and 
$\H = \Maj(\C^{c\dvc})$,
it holds that 
$\MB(\C,\H) \leq 80 \LD$.
\end{theorem}

Theorem~\ref{thm:majorities-detailed} 
will immediately follow from a result stated in the 
following section, which presents a stronger result 
where we are guaranteed not merely a small number of 
mistakes, but also a small number of points where 
the majority vote fails to have \emph{high margin}: 
namely, Theorem~\ref{thm:randomized-proper-detailed}.
Specifically, Theorem~\ref{thm:majorities-detailed} 
follows by plugging $\eps = 1/3$ 
into Theorem~\ref{thm:randomized-proper-detailed}.

\section{Online Learning with Votes of Large Margin}
\label{sec:randomized-proper}

Since $\helly(\C,\C)$ 
is sometimes large or infinite, proper learning isn't 
always viable; however, since Proposition~\ref{prop:helly-of-majorities} 
indicates $\helly(\C,\H)$ is small for $\H$ the set 
of majority votes of $O(\dvc)$ classifiers of $\C$, 
we can define predictors that, for each time $t$, 
\emph{sample} a classifier $\hat{h}_t \sim \pi_t$ for some 
distribution $\pi_t$ over $\C$, and will be correct 
with probability greater than $1/2$ against an 
adversary that only knows $\pi_t$ before selecting the 
next $(x_t,y_t)$, but does not know the specific 
$\hat{h}_t$ sampled from $\pi_t$.  
We may however 
be interested in having even greater probability of 
predicting correctly, defining a loss function that 
is $1$ if $\pi_t( h : h(x_t) \neq y_t ) > \epsilon$ 
for some given $\epsilon \in (0,1/2)$.  We would 
then be interested in bounding the number of times $t$ 
for which this occurs.

Equivalently, we can interpret this criterion in terms 
of the \emph{margin} of the majority vote classifier: 
that is,
we can think of the learning algorithm as outputting
the \emph{conditional mean} 
$\bar{h}_{t}(x) := \pi_t( h : h(x_t) = 1 )$, 
and we are interested in bounding the number 
of times $t$ where 
$| \bar{h}_{t}(x_t) - y_t | > \epsilon$.

The following is a more-detailed restatement of 
Theorem~\ref{thm:randomized-proper}, from which 
the statement of Theorem~\ref{thm:randomized-proper} 
immediately follows.

\begin{theorem}
\label{thm:randomized-proper-detailed}
Let $c=4\unifconst$ (for $\unifconst$ from Lemma~\ref{lem:dual-eps-approx}).  
For any $\eps \in (0,1/2)$,
there is an algorithm $\alg$ 
with $\H(\C,\alg) \subseteq \Vote(\C^{c\dvc/\eps^2})$ 
such that, for any $T \in \nats$, 
running $\alg$ on any $\C$-realizable sequence 
$(x_1,y_1),\ldots,(x_T,y_T)$, 
there are at most 
$\frac{8 \LD}{\eps (1-\eps/8)} \ln\!\left(\frac{8}{\eps}\right)$
times $t$ where $| \bar{h}_t(x_t) - y_t | > \eps$.
\end{theorem}

This bound will be achieved by the following algorithm, 
which takes as input $\C$ and any value $\eps \in (0,1/2)$, 
and processes a $\C$-realizable 
adversarial sequence $(x_1,y_1),\ldots,(x_T,y_T)$.

\begin{bigboxit}
0. Initialize $\Hyps = \{ (C_1,w_1) \} = \{ (\C,1) \}$, 
$\eta = \eps/8$, $m = \left\lceil \frac{2\relconst \vc}{\eps}  \log\!\left(\frac{2}{\eps}\right) \right\rceil$, $t = 1$\\
1. Repeat while $t \leq T$\\
2. \quad If $\exists h \in \Vote(\C^{c\dvc/\eps^2})$ with 
$\sup_{(x,y) \in \HighVote(\Hyps,\eps/8)} |h(x)-y| \leq \eps$\\
3. \qquad Choose $\bar{h}_{t} = h$ for some such $h$, predict $\bar{y}_t = \bar{h}_t(x_t)$\\
4. \qquad If $|\bar{y}_t - y_t| > \eps$\\
5. \qquad\quad For each $(C_i,w_i) \in \Hyps$\\
6. \qquad\qquad If $\SOA_{C_i}(x_t) \neq y_t$, set $w_i \gets \eta \cdot w_i$\\
7. \qquad\qquad $C_i \gets \{ h \in C_i : h(x_t)=y_t \}$\\
8. \qquad\qquad If $C_i = \emptyset$, remove $(C_i,w_i)$ from $\Hyps$\\
9. \qquad $t \gets t+1$\\
10.\quad Else let $\{(\tx_1,\ty_1),\ldots,(\tx_m,\ty_m)\} \subseteq \HighVote(\Hyps,\eps/8)$ be such that\\
\qquad\qquad every $h \in \C$ has 
$\frac{1}{m}\sum_{i=1}^{m} \ind[h(x_i) \neq y_i] > \eps/4$\\
11.\qquad For each $(C_i,w_i) \in \Hyps$\\
12.\qquad\quad For each $j \leq m$\\
13.\qquad\qquad If $\SOA_{C_i}(\tx_j) = \ty_j$, set $w_{ij} \gets \eta \cdot w_i$, else $w_{ij} \gets w_i$\\
14.\qquad\qquad Let $C_{ij} = \{ h \in C_i : h(\tx_j) = 1-\ty_j \}$\\
15.\qquad $\Hyps \gets \{ (C_{ij},w_{ij}) : C_{ij} \neq \emptyset \}$
\end{bigboxit}

We now present the proof of Theorem~\ref{thm:randomized-proper-detailed}.

\begin{proof}[of Theorem~\ref{thm:randomized-proper-detailed}]
The proof is similar to the proof of Theorem~\ref{thm:upper}, 
with a few important changes.
Suppose $\LD < \infty$, let $\target \in \C$ be a concept 
correct on $\{(x_t,y_t)\}_{t=1}^{T}$, 
and let $\H = \Vote(\C^{c\dvc/\eps^2})$. 
First note that, by Lemma~\ref{lem:dual-eps-approx}, 
on any given round, if the condition in Step 2 fails, 
then for every finite-support probability measure $\pi$ 
on $\C$, there exists $(x,y) \in \HighVote(\Hyps,\eps/8)$ 
such that 
$\pi( h : h(x) \neq \Maj(\Hyps)(x) ) > \eps/2$.
Thus, the existence of the sequence $(\tx_1,\ty_1),\ldots,(\tx_m,\ty_m)$ in 
Step 10 is guaranteed by Lemma~\ref{lem:eps-net}.

On each round where the condition in Step 2 holds 
but $|\bar{y}_t - y_t| > \eps$, 
we clearly have that $\target$ is correct on the $(x_t,y_t)$ 
used to update the sets $C_i$. 
Moreover, on each round where 
the condition in Step 2 fails, since $\target \in \C$, 
the defining property of $(\tx_1,\ty_1),\ldots,(\tx_m,\ty_m)$ 
in Step 10 guarantees that at least $(\eps/4) m$ 
of these $(\tx_j,\ty_j)$ have 
$\target(\tx_j) = 1-\ty_j$.
So if we think of the set $\Hyps$ as developing like a tree 
(with each element at the end of the round 
being updated from the previous round using the point $(x_t,y_t)$ on rounds of the first type, 
or else branching off of a previous element by updating with one of the $(\tx_j,1-\ty_j)$ examples 
on rounds of the second type), 
then there are a number of paths in the tree where 
all of the corresponding examples $(x,y)$ have $y=\target(x)$.
More formally, if at the beginning of a round 
there exist $q$ elements $(C,w) \in \Hyps$ with
$\target \in C$, then on rounds where the 
condition in Step 2 holds, at the end of the round 
there will still be exactly $q$ such elements in $\Hyps$ 
(i.e., none of them are removed in Step 8, since all 
contain $\target$); on the other hand, on rounds where the 
condition in Step 2 fails, then since at least $(\eps/4) m$ 
of the examples $(\tx_j,\ty_j)$ have $\target(\tx_j)=1-\ty_j$,
at the end of the round there will be at least $(\eps/4) m q$ 
elements $(C',w') \in \Hyps$ with $\target \in C'$.

Now suppose the algorithm executes the outermost loop 
at least $n$ times, and consider the total state of the algorithm 
after completing round $n$ of the outermost loop.
Let $t_n$ denote the value of $t$ 
after completing this round, and let 
$M_n$ denote the number of 
$t \in \{1,\ldots,t_n-1\}$ with 
$|\bar{y}_t - y_t| > \eps$. 
Let $N_n$ denote the number of the first $n$ rounds 
for which 
the condition in Step 2 fails.
Let $W_n$ denote the total of the weights in $\Hyps$ after round $n$, and define $W_0=1$.

Applying the above argument inductively, 
we have that after completing round $n$ of the outermost loop, 
there are at least $((\eps/4) m)^{N_n}$ elements 
$(C,w) \in \Hyps$ with $\target \in C$.
Moreover, we note that any $(C,w) \in \Hyps$ 
with $\target \in C$ must have 
$w \geq \eta^{\LD}$, since on any 
sequence of labeled examples $(x'_i,\target(x'_i))$,
there are at most $\LD$ times $i$ with 
$\SOA_{\C_{\{(x'_j,\target(x'_j)) : j < i\}}}(x'_i) \neq \target(x'_i)$, 
as established by \citet*{littlestone:88}.
It follows that after completing round $n$ we have 
\begin{equation*}
W_n \geq \left( (\eps/4) m \right)^{N_n} \cdot \eta^{\LD}.
\end{equation*}

On the other hand, on every round $n' \leq n$ 
where 
the condition in Step 2 holds
but $|\bar{y}_t - y_t| > \eps$ in Step 4 
(for $t = t_{n'}-1$), 
note that it cannot be that $(x_t,y_t) \in \HighVote(\Hyps,\eps/8)$, 
since this would contradict 
either $|\bar{y}_t - y_t| > \eps$ 
or the defining property of $\bar{h}_t$.
Thus, since $(x_t,y_t) \notin \HighVote(\Hyps,\eps/8)$, 
at least $\eps/8$ fraction of the 
total weight is multiplied by $\eta$: 
that is, 
$W_{n'} \leq W_{n'-1} \left( \eta (\eps/8) + (1-(\eps/8)) \right)$.
Furthermore, on every round where 
the condition in Step 2 fails, 
for each $j \leq m$, 
since $(\tx_j,\ty_j) \in \HighVote(\Hyps,\eps/8)$, 
it must be that at least $1-(\eps/8)$ fraction of the 
total weights $w_i$ from the previous round 
will have $w_{ij} = \eta \cdot w_{i}$.
Together with the growth by a factor of $m$ from branching, 
we have that on such rounds $n' \leq n$, 
$W_{n'} \leq W_{n'-1} m \left( \left(1-\frac{\eps}{8}\right) \eta + \frac{\eps}{8} \right)$.
By induction we have that, after round $n$, 
\begin{equation*}
W_n \leq \left( \eta \frac{\eps}{8} + 1-\frac{\eps}{8}\right)^{M_n} m^{N_n} \left( \left(1-\frac{\eps}{8}\right)\eta + \frac{\eps}{8} \right)^{N_n}.
\end{equation*}

Combining the upper and lower bounds, we have
\begin{equation*}
\left( \frac{\eps}{4} m \right)^{N_n} \cdot \eta^{\LD} 
\leq \left( \eta \frac{\eps}{8} + 1-\frac{\eps}{8}\right)^{M_n} m^{N_n} \left( \left(1-\frac{\eps}{8}\right)\eta + \frac{\eps}{8} \right)^{N_n}.
\end{equation*}
Plugging in $\eta = \eps/8$ we have 
\begin{align*}
\left( \frac{\eps}{4} m \right)^{N_n} \cdot \left(\frac{\eps}{8}\right)^{\LD} 
& \leq \left( 1 - \frac{\eps}{8}\!\left( 1 - \frac{\eps}{8} \right) \right)^{M_n} \left( \left(1-\frac{\eps}{16}\right) \frac{\eps}{4} m \right)^{N_n} 
\\ & \leq \exp\!\left\{ - \frac{\eps}{8}\left( 1 - \frac{\eps}{8} \right) M_n \right\} \cdot \left( \left(1-\frac{\eps}{16}\right) \frac{\eps}{4} m \right)^{N_n}. 
\end{align*}
Taking logarithms of both sides and simplifying, we have
\begin{equation*}
M_n  
+ 
N_n \frac{8}{\eps \left( 1 - \eps/8 \right)} \ln\!\left( \frac{1}{1-\eps/16} \right) 
\leq \frac{8 \LD}{\eps (1-\eps/8)} \ln\!\left(\frac{8}{\eps}\right).
\end{equation*}
This has two important implications. 
First, since we always have $n = N_n + t_n-1$, 
and $t_n \leq T+1$ while the above inequality implies 
$N_n \leq \LD \frac{\ln(8/\eps)}{\ln(1/(1-\eps/16))} < \infty$, 
we may conclude that the algorithm will terminate after 
a finite number of rounds.
Second, the above inequality further implies 
$M_n \leq \frac{8\LD}{\eps (1-\eps/8)} \ln\!\left(\frac{8}{\eps}\right)$ for all rounds $n$ 
in the algorithm, so that this also bounds the total number 
of times $t \leq T$ with $|\bar{y}_t - y_t| > \eps$.
This completes the proof.
\end{proof}

\begin{remark}
\label{rem:general-conversion}
We also remark that the above algorithm can actually 
be executed with \emph{any} online learning algorithm 
$\alg$ in place of $\SOA$, resulting in a conversion 
to a method using votes of $O(\dvc/\eps^2)$ concepts 
in $\C$, and having a number of rounds with 
$|\bar{h}_t(x_t) - y_t| > \epsilon$ at most 
$O\!\left(\MB(\C,\alg,T) \frac{1}{\eps} \log \frac{1}{\eps}\right)$.
\end{remark}

\section{Near-Optimal Agnostic Online Learning with Randomized Proper Predictors}
\label{sec:agnostic}

This section presents the details of the 
algorithm and proof for Theorem~\ref{thm:agnostic}.
Recall the statement of the theorem, as follows.

\noindent \textbf{Theorem~\ref{thm:agnostic} (Restated)}~
For any $T \in \nats$, 
there is an algorithm $\alg$ 
with $\H(\alg,T) \subseteq \Vote(\C^{m})$, 
where $m = O\!\left(\frac{\dvc T}{\LD \log(T/\LD)}\right)$, 
such that for any sequence 
$(x_1,y_1),\ldots,(x_T,y_T) \in \X \times \Y$, 
\begin{equation*} 
\sum_{t=1}^{T} \left| \bar{h}_t(x_t) - y_t \right| - \min_{h \in \C} \sum_{t=1}^{T} \ind[ h(x_t) \neq y_t ] = O\!\left( \sqrt{\LD T \log(T/\LD)} \right).
\end{equation*}

Before presenting the proof, let us first note that the 
algorithm we propose is, to a large extent, constructive, 
in the sense that on each round $t$ it constructs a 
probability measure $p_t$ on $\C$, based on 
$(x_1,y_1),\ldots,(x_{t-1},y_{t-1})$, 
so that $\bar{h}_t(x_t) = p_t( h : h(x_t) = 1 )$.
Thus, we may regard this method as a 
\emph{randomized proper} learning algorithm, 
in the sense that on each round, if we draw a 
random $h_t \sim p_t$ (independently, given the data 
sequence) then 
$\E \sum_{t=1}^{T} \ind[ h_t(x_t) \neq y_t ] 
= \sum_{t=1}^{T} \left| \bar{h}_t(x_t) - y_t \right|$, 
so that the \emph{expected regret} of these sampled 
classifiers is at most 
$O\!\left( \sqrt{ \LD T \log(T/\LD) } \right)$.

As discussed above, this result can be interpreted as a more 
constructive version of a result of
\citet*{rakhlin2015online}.  The result 
in Theorem~\ref{thm:agnostic} is also more general, 
as it removes additional restrictions on $\C$ 
required by \citet*{rakhlin2015online}.

Turning now to the task of proving Theorem~\ref{thm:agnostic}, 
we will rely on the following Lemma from \citet*{ben2009agnostic}.
For any sequence $x_1,\ldots,x_t$, 
abbreviate $x_{1:t} = (x_1,\ldots,x_t)$.

\begin{lemma}[{\citealp*[][Lemma 12]{ben2009agnostic}}]\label{lem:benagnostic}
For any concept class $\C$ of Littlestone dimension $\LD$, 
for any $T \in \nats$, there exists a family 
$\Experts_{T} := \{ \expert_{I} : I \subseteq \{1,\ldots,T\}, |I| \leq \LD \}$ 
of functions $\X^{*} \to \Y$ such that, 
for every $x_1,\ldots,x_T \in \X$,
\begin{equation*}
\left\{ (h(x_1),\ldots,h(x_T)) : h \in \C \right\} 
= \left\{ (\expert_{I}(x_1),\expert_{I}(x_{1:2}),\ldots,\expert_{I}(x_{1:T})) : \expert_{I} \in \Experts_{T} \right\}.
\end{equation*}
\end{lemma}

Specifically, \citet*{ben2009agnostic} 
construct this family $\Experts_{T}$ by letting 
$\expert_{I}(x_{1:t})$ be the prediction 
of $\SOA_{H(t,I)}(x_t)$, 
where $H(0,I) = \C$, and $H(t,I)$ is inductively 
defined as  
$H(t,I) = \{ h \in H(t-1,I) : h(x_t) = 1-\SOA_{H(t-1,I)}(x_t)) \}$ 
if $t \in I$ and 
$\{ h \in H(t-1,I) : h(x_t) = 1-\SOA_{H(t-1,I)}(x_t)) \} \neq \emptyset$,
and $H(t,I) = H(t-1,I)$ otherwise.
The property guaranteed by Lemma~\ref{lem:benagnostic} 
then follows from the $\LD$ mistake bound of 
\citet*{littlestone:88} for $\SOA$: that is, 
the $\expert_{I}$ that agrees with $h$ on 
$x_{1:T}$ simply takes $I$ as the (at most $\LD$) 
times when $\SOA_{H}$ 
would make mistakes on $(x_1,h(x_1)),\ldots,(x_T,h(x_T))$, 
given that it updates $H \gets \{ h' \in H : h'(x_t) = h(x_t) \}$ after each mistake $(x_t,h(x_t))$.

We will also make use of a classic result for learning 
from expert advice for the absolute loss 
\citep*{vovk:90,vovk:92,littlestone:94,cesa-bianchi:97,kivinen:99,singer:99}; 
see Theorem 2.2 of \citet*{cesa-bianchi:06}.

\begin{lemma}
\label{lem:experts}
\citep*[][Theorem 2.2]{cesa-bianchi:06} 
For any $N,T \in \nats$ and $f_1,\ldots,f_N$ 
functions $\X^* \to [0,1]$, 
letting $\eta = \sqrt{(8/T)\ln(N)}$, 
for any $(x_1,y_1),\ldots,(x_T,y_T) \in \X \times [0,1]$, 
letting $w_{0,i}=1$ and 
$w_{t,i} = e^{- \eta \sum_{s \leq t} |f_i(x_{1:s})-y_s|}$ for each $t \leq T$, $i \leq N$, 
letting $\bar{f}_t(x_{1:t},y_{1:(t-1)}) = \sum_{i} w_{t-1,i} f_i(x_{1:t}) / \sum_{i'} w_{t-1,i'}$, 
it holds that 
\begin{equation*}
\sum_{t = 1}^{T} \left| \bar{f}_t(x_{1:t},y_{1:(t-1)}) - y_t \right| - \min_{1 \leq i \leq N} \sum_{t=1}^{T} \left| f_i(x_{1:t}) - y_t \right| \leq \sqrt{(T/2)\ln(N)}.
\end{equation*}
\end{lemma}

We are now ready for the proof of 
Theorem~\ref{thm:agnostic}.

\begin{proof}[of Theorem~\ref{thm:agnostic}]
If $T < 10 \LD$, the regret bound is $\geq T$ 
(for appropriate constants), which trivially holds, 
so let us suppose $T \geq 10 \LD$.
Fix a value $\eps = \sqrt{(\LD/T)\log(eT/\LD)}$  
and let $\alg$ be the algorithm guaranteed by 
Theorem~\ref{thm:randomized-proper-detailed}: 
that is, for any $\C$-realizable sequence 
$(x_1,y_1),\ldots,(x_T,y_T)$, for each $t \leq T$,  
$\alg$ proposes a finite-support 
probability measure $\pi_t$ on $\C$ 
(namely, a uniform distribution on at most 
$O(\dvc / \eps^2)$ elements of $\C$) 
defined based solely on $(x_1,y_1),\ldots,(x_{t-1},y_{t-1})$, 
and there are at most $O\!\left(\frac{\LD}{\eps}\log\frac{1}{\eps}\right)$ 
times $t$ for which 
$| \pi_{t}( h : h(x_t) = 1 ) - y_t | > \eps$.

Now consider running $\alg$ for the sequence 
$(x_1,\expert_{I}(x_1)),(x_2,\expert_{I}(x_{1:2})),\ldots,(x_T,\expert_{I}(x_{1:T}))$ for $\expert_{I} \in \Experts_{T}$, where $\Experts_{T}$ is from Lemma~\ref{lem:benagnostic}. 
In particular, Lemma~\ref{lem:benagnostic} 
guarantees this is a $\C$-realizable sequence.
Let $\pi^{I}_{t}$ denote the corresponding 
probability measure 
that would be proposed on each round $t$, and let 
$\bar{h}^{I}_{t}(x) = \pi^{I}_{t}( h : h(x) = 1 )$.

Let $\target = \argmin_{h \in \C} \sum_{t=1}^{T} \ind[ h(x_t) \neq y_t ]$.
By Lemma~\ref{lem:benagnostic}, there exists 
$\expert_{I^*} \in \Experts_{T}$ with 
$\expert_{I^*}(x_{1:t}) = \target(x_t)$ 
for every $t \leq T$.
We therefore have 
\begin{equation*} 
\sum_{t=1}^{T} \ind\!\left[ \left| \bar{h}^{I^*}_{t}(x_t) - \target(x_t) \right| > \eps \right] = O\!\left( \frac{\LD}{\eps}\log\frac{1}{\eps} \right).
\end{equation*}
In particular, this immediately implies 
\begin{equation}
\label{eqn:Istar-bound}
\sum_{t=1}^{T} \left| \bar{h}^{I^*}_{t}(x_t) - \target(x_t) \right| \leq \eps T + O\!\left( \frac{\LD}{\eps}\log\frac{1}{\eps} \right).
\end{equation}

Let $\I_{T} = \{ I \subseteq \{1,\ldots,T\} : |I| \leq \LD \}$.
We will now treat the predictors 
$\bar{h}^{I}_{t}$, $I \in \I_{T}$, 
as \emph{experts} in the traditional framework of 
learning from expert advice for the absolute loss.
Let $N = \binom{T}{\leq \LD} = \sum_{j=0}^{\LD} \binom{T}{j}$.
Following Lemma~\ref{lem:experts}, 
let $\eta = \sqrt{(8/T)\ln(N)}$, 
and for each $I \in \I_{T}$ let 
$w_{0,I} = 1$ and $w_{t,I} = e^{-\eta \sum_{s \leq t} |\bar{h}^{I}_{s}(x_s) - y_s|}$ for each $t \leq T$.
Finally, define 
$\bar{h}'_t(x) = \sum_{I \in \I_T} w_{t-1,I} \bar{h}^{I}_{t}(x) / \sum_{I \in \I_T} w_{t-1,I}$.

By Lemma~\ref{lem:experts}, we have 
\begin{align*}
\sum_{t=1}^{T} \left| \bar{h}'_t(x_t) - y_t \right| 
& \leq \sqrt{(T/2)\ln(N)} + \min_{I \in \I_{T}} \sum_{t=1}^{T} \left| \bar{h}^{I}_{t}(x_t) - y_t \right| 
\\ & \leq \sqrt{\LD (T/2)\ln\!\left( \frac{e T}{\LD} \right)} +
\sum_{t=1}^{T} \left| \bar{h}^{I^*}_{t}(x_t) - y_t \right|. 
\end{align*}
By the triangle inequality, the rightmost expression 
is at most
\begin{equation*}
\sqrt{\LD (T/2)\ln\!\left( \frac{e T}{\LD} \right)} +
\sum_{t=1}^{T} \left| \bar{h}^{I^*}_{t}(x_t) - \target(x_t) \right| + \sum_{t=1}^{T} \left| \target(x_t) - y_t \right|.
\end{equation*}
Combining this with \eqref{eqn:Istar-bound} 
and plugging in $\eps = \sqrt{(\LD/T)\log(eT/\LD)}$
yields 
\begin{equation*} 
\sum_{t=1}^{T} \left| \bar{h}'_t(x_t) - y_t \right| 
- \sum_{t=1}^{T} \left| \target(x_t) - y_t \right|
=
O\!\left( \sqrt{\LD T \log\!\left(\frac{T}{\LD}\right) } \right).
\end{equation*}
The claimed regret now follows by noting that 
$\bar{h}'_t(x) = p'_t( h : h(x) = 1 )$ by 
choosing $p'_t = \sum_{I \in \I_T} \frac{w_{t-1,I}}{\sum_{I' \in \I_T} w_{t-1,I'}} \pi^{I}_{t}$.
Since each $\pi^{I}_{t}$ is a finite-support 
probability measure on $\C$, 
$p'_t$ is also a finite-support probability measure 
on $\C$.

To conclude, we note that the probability measures $p'_t$ 
in the above proof may have supports of size  
up to 
$\tilde{O}\!\left(\dvc \left(\frac{eT}{\LD}\right)^{\LD+1} \right)$.
However, we can apply the same sparsification 
argument used in previous sections: that is, 
since $p'_t$ has finite support, 
we can apply the 
classic result on the size of $\eps$-approximating 
sets (Lemma~\ref{lem:vc-bound} of Appendix~\ref{sec:vc-bounds}), 
which implies there exists $\bar{h}_t \in \Vote(\C^{m})$ 
with $m = O\!\left(\frac{\dvc}{\eps^2}\right)$
such that 
$\sup_{x} | \bar{h}_t(x) - p'_t(h : h(x)=1) | \leq \eps$,
so that using 
$\bar{h}_t$
on each round, rather than $\bar{h}'_t$,
by the triangle inequality we have 
\begin{align*}
& \sum_{t=1}^{T} \left| \bar{h}_t(x_t) - y_t \right| 
- \min_{h \in \C} \sum_{t=1}^{T} \left| h(x_t) - y_t \right| 
\\ & \leq \eps T + \sum_{t=1}^{T} \left| \bar{h}'_t(x_t) - y_t \right| - \min_{h \in \C} \sum_{t=1}^{T} \left| h(x_t) - y_t \right|
= \eps T + O\!\left( \sqrt{\LD T \log(T/\LD)} \right).
\end{align*}
Taking $\eps = \sqrt{(\LD/T)\log(T/\LD)}$ then retains 
the regret bound 
$O\!\left( \sqrt{\LD T \log(T/\LD)} \right)$ 
required by Theorem~\ref{thm:agnostic}, 
with
$\bar{h}_t \in \Vote(\C^m)$ 
for $m = O\!\left( \frac{\dvc T}{\LD \log(T/\LD)}\right)$.
\end{proof}

\bibliography{learning}

\appendix

\section{Uniform Convergence Bounds for VC Classes}
\label{sec:vc-bounds}

\begin{lemma}
\label{lem:vc-bound}
There is a universal constant $\vcconst$ 
such that, for any set $\Z$ and any finitely-supported 
probability measure $P$ on $\Z$, 
for any set $\F$ of functions $\Z \to \{0,1\}$, 
for $m \in \nats$ and $\eps \in (0,1]$, 
if $m \geq \frac{\vcconst \vc(\F)}{\epsilon}\log\frac{1}{\epsilon}$, 
then there exists $z_1,\ldots,z_m \in \Z$ 
such that every $f \in \F$ with 
$P( z : f(z) = 1 ) \geq \frac{2}{3}\eps$ 
has 
$\frac{1}{m}\sum_{t=1}^{m} f(z_t) > \frac{\eps}{2}$.
Moreover, if $m \geq \frac{\vcconst \vc(\F)}{\eps^2}$, 
then there exists $z_1,\ldots,z_m \in \Z$ 
such that every $f \in \F$ satisfies 
$\left| P( z : f(z) = 1 ) - \frac{1}{m} \sum_{t=1}^{m} f(z_t) \right| < \eps$.
\end{lemma}
\begin{proof}
The first claim follows from the 
relative uniform convergence bounds of \citet*{vapnik:74}, 
which hold without further restrictions on $\F$ since 
$P$ has finite support; see Theorem 4.4 of \citet*{vapnik:98}.
Specifically, 
Theorem 4.4 of \citet*{vapnik:98} guarantees that 
(for an appropriate constant $\vcconst$), 
for $(Z_1,\ldots,Z_m) \sim P^m$, 
with nonzero probability, 
every $f \in \F$ with $P( z : f(z) = 1 ) \geq \frac{2}{3} \eps$ 
has $\frac{1}{m} \sum_{t=1}^{m} f(Z_t) > \frac{\eps}{2}$, 
which then implies there exist at least one such 
sequence $z_1,\ldots,z_m$ satisfying the claim.
Similarly, for the second claim, 
by the classic uniform convergence guarantees 
based on the chaining argument 
(see \citealp*{talagrand:94,van-der-Vaart:96}), 
which again hold without further restrictions to $\F$ 
because $P$ has finite support,
it holds that, 
if $m \geq \frac{\vcconst \vc(\F)}{\eps^2}$ 
(for an appropriate constant $\vcconst$), 
for $(Z_1,\ldots,Z_m) \sim P^m$, 
with nonzero probability, every $f \in \F$ satisfies  
$\left| P( z : f(z) = 1 ) - \frac{1}{m} \sum_{t=1}^{m} f(z_t) \right| < \eps$.
In particular, this implies there must exist 
at least one such sequence $z_1,\ldots,z_m$ satisfying 
the claim.
\end{proof}

\end{document}